\documentclass[11pt]{article}

\usepackage{amsmath}

\usepackage{bbm}
\usepackage{bm}
\usepackage{multirow}
\usepackage{enumitem}
\usepackage{marginnote}
\usepackage{graphicx}
\usepackage{mdframed}
\usepackage{mathcomp}
\usepackage{stmaryrd}
\usepackage{mathtools}
\usepackage{xr}
\usepackage{nicefrac}       
\usepackage{hyperref}       

\usepackage{microtype}      

\usepackage{amsthm}
\usepackage{amssymb}
\usepackage{amsfonts}       
\newtheorem{theorem}{Theorem}
\newtheorem{lemma}[theorem]{Lemma}

\newtheorem{proposition}[theorem]{Proposition}
\newtheorem{definition}[theorem]{Definition}
\usepackage{xcolor}
\usepackage[letterpaper, left=1in, right=1in, top=1in, bottom=1in]{geometry}

\newcommand{\bell}{\bm{\ell}}

\renewcommand{\v}{\bm{v}}

\newcommand{\cU}{\mathcal{U}}

\newcommand{\cB}{\mathcal{B}}    

\newcommand{\cT}{\mathscr{T}}

\newcommand{\wtilde}{\widetilde}

\newcommand{\loss}{\ell}
\DeclareBoldMathCommand{\vloss}{\loss}
\DeclareBoldMathCommand{\grad}{g}
\DeclareBoldMathCommand{\fakegrad}{\mathring{\bm{g}}}
\DeclareBoldMathCommand{\e}{e}
\DeclareBoldMathCommand{\p}{p}
\DeclareBoldMathCommand{\u}{u}
\DeclareBoldMathCommand{\w}{w}
\DeclareBoldMathCommand{\x}{x}
\DeclareBoldMathCommand{\l}{l}
\DeclareBoldMathCommand{\vzero}{0}
\let\top\intercal

\newcommand{\reals}{\mathbb{R}}

\renewcommand{\x}{\bm{x}}   
\newcommand{\z}{\bm{z}}
\newcommand{\y}{\bm{y}} 
\newcommand{\g}{\bm{g}}

\DeclareMathOperator{\epi}{epi}

\DeclareMathOperator*{\argmin}{argmin}
\DeclareMathOperator*{\argmax}{argmax}

\newcommand{\G}{\mathbf{G}}
\newcommand{\cK}{\mathcal{K}}

\newcommand{\dom}{\operatorname{dom}}

\newcommand{\cI}{\mathcal{I}}
\newcommand{\cA}{\mathcal{A}}
\renewcommand{\cT}{\mathcal{T}}

\newcommand{\nn}{\nonumber}

\newcommand{\inner}[2]{\langle#1,#2\rangle}

\newmdtheoremenv{condition}{Condition}

\newcommand{\vertiii}[1]{{\left\vert\kern-0.25ex\left\vert\kern-0.25ex\left\vert #1 
    \right\vert\kern-0.25ex\right\vert\kern-0.25ex\right\vert}}

\makeatletter
\newcommand*\bcdot{\mathpalette\bigcdot@{.5}}
\newcommand*\bigcdot@[2]{\mathbin{\vcenter{\hbox{\scalebox{#2}{$\m@th#1\bullet$}}}}}

\newlength\myindent
\usepackage{algorithm}
\usepackage{algorithmic}
\newcommand{\algcomment}[1]{  \textcolor{blue!70!black}{\footnotesize{\texttt{{//
					#1}}}}}

\usepackage{comment}
\usepackage{bm}

\usepackage{bbding}
\usepackage{booktabs}
\usepackage{footnote}
\makesavenoteenv{tabular}
\makesavenoteenv{table}

\usepackage{notation}
\usepackage{bm}
\mathtoolsset{showonlyrefs}
\usepackage{multirow}
\usepackage{microtype}

\newcommand{\K}{\mathcal{C}}
\newcommand{\cO}{\mathcal{O}}

\newcommand{\cE}{\mathcal{E}}  
  
\renewcommand{\P}{\mathbb{P}}

\newcommand{\FTL}{\mathsf{FTL}}   
\newcommand{\FTSL}{\mathsf{FTSL}}
\newcommand{\od}{\mathsf{FG}}   
\renewcommand{\bell}{\bm{\ell}}

\renewcommand{\dom}{\mathrm{dom}}
\newcommand{\ball}{\mathbb{B}}
	\renewcommand{\epi}{\mathrm{epi}}
\renewcommand{\c}{\bm{c}}
\usepackage{todonotes}
\usepackage{colortbl}

\makeatletter
\g@addto@macro\bfseries{\boldmath}
\makeatother

\newtheorem{remark}[theorem]{Remark}

\newcommand*\accentfontxheight[1]{%
	\fontdimen5\ifx#1\displaystyle
	\textfont
	\else\ifx#1\textstyle
	\textfont
	\else\ifx#1\scriptstyle
	\scriptfont
	\else
	\scriptscriptfont
	\fi\fi\fi3
}
\let\wtilde\undefined

\makeatletter
\newcommand*\wtilde[1]{\mathpalette\wthelpers{#1}}

\newcommand*\wthelpers[2]{%
	\hbox{\dimen@\accentfontxheight#1%
		\accentfontxheight#11.2\dimen@
		$\m@th#1\widetilde{#2}$%
		\accentfontxheight#1\dimen@
	}%
}

\begin{document}

	\title{Exploiting the Curvature of Feasible Sets for Faster Projection-Free Online Learning}
	\author{{\bf Zakaria Mhammedi}    \\
		Massachusetts Institute of Technology \\ 
		\texttt{mhammedi@mit.edu}}

	\maketitle
	
	\begin{abstract}
		In this paper, we develop new efficient projection-free algorithms for Online Convex Optimization (OCO). Online Gradient Descent (OGD) is an example of a classical OCO algorithm that guarantees the optimal $O(\sqrt{T})$ regret bound. However, OGD and other projection-based OCO algorithms need to perform a Euclidean projection onto the feasible set $\K\subset \reals^d$ whenever their iterates step outside $\K$. For various sets of interests, this projection step can be computationally costly, especially when the ambient dimension is large. This has motivated the development of projection-free OCO algorithms that swap Euclidean projections for often much cheaper operations such as Linear Optimization (LO).
		However, state-of-the-art LO-based algorithms only achieve a suboptimal $O(T^{3/4})$ regret for general OCO. In this paper, we leverage recent results in parameter-free Online Learning, and develop an OCO algorithm that makes two calls to an LO Oracle per round and achieves the near-optimal $\wtilde{O}(\sqrt{T})$ regret whenever the feasible set is strongly convex. We also present an algorithm for general convex sets that makes $\wtilde O(d)$ expected number of calls to an LO Oracle per round and guarantees a $\wtilde O(T^{2/3})$ regret, improving on the previous best $O(T^{3/4})$. We achieve the latter by approximating any convex set $\K$ by a strongly convex one, where LO can be performed using $\wtilde {O}(d)$ expected number of calls to an LO Oracle for $\K$.
	\end{abstract}

	\section{Introduction}
	In this paper, we are interested in developing efficient algorithms for Online Convex Optimization (OCO). In OCO, at each round $t\geq 1$, a learner (the algorithm) outputs a point $\w_t$ in a closed convex set $\K\subseteq \reals^d$. Then, the environment chooses a convex loss function $f_t\colon \K \rightarrow \reals$ that can depend on $\w_t$ and the past history $(\w_s, f_s)_{s<t}$, and the learner suffers a loss $f_t(\w_t)$. The goal of the learner is to minimize the regret $\sum_{t=1}^T (f_t(\w_t)-f_t(\w))$ against any comparator $\w \in \K$ not known to the learner in advance. OCO generalizes the offline and stochastic optimization settings, and any algorithm/guarantee for OCO is readily transferred to these settings via online-to-batch conversion techniques (see e.g.~\cite{cesa2004, shalev2011, cutkosky2019}). OCO has many ML applications beyond offline and stochastic optimization. For example, OCO algorithms have been successfully used in controls \cite{agarwal2019logarithmic, foster2020logarithmic,zhang2021adversarial} and reinforcement learning \cite{foster2021statistical}. See \cite{hazan2016introduction} for examples of problems that can be modeled via OCO.
	
	Online Gradient Descent (OGD) \cite{zinkevich2003} is an example of an OCO algorithm that guarantees the optimal $O(\sqrt{T})$ regret for any sequence of adversarially-chosen losses $(f_t)$. Remarkably this worst-case regret does not deteriorate with the dimension $d$ of the ambient space. On the other hand, the computational complexity of OGD (and other projection-based algorithms) can become prohibitively large when $d$ is large due to the complexity of the Euclidean projections that need be performed whenever the iterates of the algorithm step outside the feasible set $\K$ \cite{jaggi2013}. This computational challenge has motivated the search for efficient projection-free OCO algorithms. The vast majority of projection-free algorithms are based on the Frank-Wolfe (FW) algorithm \cite{frank1956} that was initially used for offline smooth optimization on polyhedral sets. The FW algorithm swaps potentially expensive projections onto the set $\K$ for Linear Optimization on $\K$, which can often be performed efficiently \cite{jaggi2013}. The algorithm was later extended to the OCO setting with linear losses \cite{kalai2005} and subsequently to general online and stochastic optimization; see e.g.~\cite{hazan2012, hazana16, hazan2020}.
	
	The computational advantage that FW-style algorithms bring comes at the cost of a worse performance as measured by regret. In fact, for general convex sets and convex functions, the best-known regret bound of any FW-style algorithm in OCO is of order $O(T^{3/4})$ \cite{hazan2012}, which is much worse than the optimal $O(\sqrt{T})$ regret bound guaranteed by, for example, OGD. As we detail in the related works paragraph below, many works have leveraged additional structure of the OCO setting of interest, such as the smoothness/strong convexity of the losses or the smoothness of the set $\K$ itself, to improve on the $O(T^{3/4})$ regret. In this paper, we exploit the curvature (strong convexity) of the feasible set to improve over the state-of-the-art regret bound of FW-style projection-free algorithms.
	
	\paragraph{Contributions.} We present a new projection-free algorithm for OCO that achieves the near-optimal $\wtilde{O}(\sqrt{T})$ regret bound when the set $\K$ is strongly convex (see \S\ref{sec:prelim} for a definition and examples). Our new algorithm makes at most two calls/round to a Linear Optimization Oracle (LOO) on the set $\K$. Among algorithms that use a similar number of LOO calls, our algorithm is the first to guarantee a $\wtilde O(\sqrt{T})$ regret bound on strongly convex sets and general convex losses, improving on the previous best $O(T^{2/3})$ by \cite{wan2021projection}. Our solution relies on recent techniques in parameter-free Online Learning.
	
	For our second contribution, we present an algorithm for general convex sets and convex losses that makes $\wtilde O(d)$ expected \# of calls per round to an LOO and achieves a $O(T^{2/3})$ regret bound (improving over the previous best $O(T^{3/4})$ for this setting). To achieve this result, we show how any convex set can be approximated by a strongly convex one, where our algorithm from the previous paragraph can be applied. Here, a natural trade-off between set approximation and regret arises. Optimal tuning of the parameters leads to our $\wtilde O(T^{2/3})$ regret bound (see Tab.~\ref{tab:result} for a results' summary). One of the major hurdles we overcome involves showing that Linear Optimization on the approximate set can be done efficiently with only $\wtilde{O}(d)$ expected \# of calls to an LOO for the original set $\K$ (this may be of independent interest).
	
	\paragraph{Related Works.} Our paper adds to the long line of research into efficient projection-free algorithms for online, offline, and stochastic optimization \cite{frank1956,kalai2005, hazan2008sparse,hazan2012,jaggi2013,hazana16,garber2016faster,kerdreux2020accelerating,hazan2020, Bomze2021}. \cite{hazan2012} were the first to obtain the $O(T^{3/4})$ regret for general OCO. This regret bound has been improved to $O(T^{2/3})$ for smooth [resp.~strongly convex] losses by \cite{hazan2020} [resp.~\cite{kretzu2021}] using an algorithm that makes a single call to an LOO per round. The geometry of the set $\K$ has also been exploited in the design of efficient projection-free algorithms. \cite{garber2016} presented an algorithm that achieves a $O(\sqrt{T})$ regret when the feasible set is a polytope while making a single call to an LOO per round. \cite{mahdavi2012} and \cite{levy2019} exploited the smoothness of the feasible set to achieve better regret bounds. In particular, when the set $\K$ is 
	smooth, \cite{levy2019} present a projection-free algorithm that guarantees the optimal $O(\sqrt{T})$ regret. However, the algorithms of \cite{levy2019} and \cite{mahdavi2012} use a Separation/Membership Oracle instead of an LOO. Along the same lines, \cite{mhammedi2021efficient,garber2022new} recently presented algorithms that achieve the optimal $O(\sqrt{T})$ regret for general OCO using $\wtilde O(1)$ calls per round to a Separation/Membership Oracle. Algorithms that use Separation/Membership Oracles can be viewed as complementary to those that use an LOO (see Appendix \ref{sec:oracles}). In this paper, our focus is on algorithms that use an LOO. 
	
	We also exploit the geometry of the feasible set to achieve improved regret bounds. In particular, we consider the case where $\K$ is strongly convex. Strong convexity of a set has been leveraged in the offline optimization setting by \cite{garber2015faster, abernethy2018faster} to achieve the fast $O(1/T^2)$ rate for smooth, strongly convex objectives. 
	Their algorithms make $O(T)$ calls to an LOO after $T$ rounds. \cite{kerdreux2021projection} proved fast rates for FW on strictly convex sets. Like in this paper, \cite{wan2021projection} also considered OCO with strongly convex feasible sets. However, their algorithm achieves the suboptimal $O(T^{2/3})$ regret for general losses. To guarantee a $O(\sqrt{T})$ regret bound, they require strong convexity of the losses---something that we do not require in this work. Finally, \cite{huang2016following} presented an algorithm that achieves a logarithmic regret in Online Linear Optimization whenever the norm of the sum of the loss vectors grows linearly with the number of rounds. A similar condition was exploited in the works of \cite{abernethy2017frank,molinaro2020} to achieve a logarithmic regret. In this paper, we impose no such condition on the growth of the sum of the subgradients. In Table \ref{tab:result}, we summarize the existing results just discussed and contrast them with ours. 
	
	To build our algorithms, we put together various techniques from OCO an convex geometry. Our main algorithm relies crucially on a comparator adaptive subroutine---an algorithm whose regret scales with the norm of the comparator \cite{mcmahan2012no,mcmahan2014unconstrained, orabona2016coin,foster2017parameter,cutkosky2018black,mhammedi2020}. The idea of sleeping experts \cite{adamskiy2012closer,gaillard2014second} will also be central to the design of our algorithm. 
	When it comes to applying our result for general convex sets, we use a recent result by \cite{molinaro2020} that approximates a bounded convex set by a strongly convex one. Our contribution here is an explicit, efficient algorithm for LO over the approximate set.
	
		\begin{table}
		\fontsize{8}{9}\selectfont
		\centering
		\begin{tabular}{llllll}
			\hline
			Loss Functions &  Feasible Set  & Oracle Type  & \# Oracle Calls   &  \multicolumn{2}{l}{Regret Bound}       \\
			& $\K \subset \reals^d$ & & per Round  & \\
			\hline 
			\hline
			{\color{blue} 	General	}  & 	{\color{blue} General}  & 	{\color{blue} Linear Optimization} & 	{\color{blue} 1}   &  	{\color{blue} $T^{3/4}$ }&	{\color{blue}  \cite{hazan2012,garber2022new}  }\\
			{\color{blue} 	General	}  &	{\color{blue}  General  }&	{\color{blue}  Linear Optimization }& 	{\color{blue} $d  \ln (d\kappa T)$ }  & 	{\color{blue} \bf $(\kappa T)^{2/3}$} &	{\color{blue}  Theorem \ref{thm:master2} } \\
			{\color{blue} General }& 	{\color{blue} $\mu$-strongly conv. }&	{\color{blue}  Linear Optimization } &	{\color{blue}  1 } &  	{\color{blue} $\mu^{-1} T^{2/3}$ }& 	{\color{blue} \cite{wan2021projection}}\\ 
			{\color{blue} General}  & 	{\color{blue}  $\mu$-strongly conv.} &	{\color{blue}  Linear Optimization } & 	{\color{blue} 2} &  	{\color{blue}\bf $\sqrt{\mu^{-1}T\ln T}$ }& 	{\color{blue}  Theorem \ref{thm:master}} \\ 
			$\alpha$-strongly conv. & General  & Linear Optimization & 1   &  $ \alpha^{-1}  T^{2/3}$ & \cite{kretzu2021} \\
			$\alpha$-strongly conv. &	$\mu$-strongly conv. &   Linear Optimization  & 1 & ${\alpha}{\mu^{-2}}\sqrt{T}$ & \cite{wan2021projection} \\ 
			Smooth & General & Linear Optimization & 1 & $(\beta+\sqrt{d})T^{2/3}$ & \cite{hazan2020} \\ 
			Linear  &  General & Linear Optimization & 1 & $\sqrt{T}$ & \cite{kalai2005}  \\
			Linear with & $\mu$-strongly conv. & Linear Optimization& 1 & $({\mu L})^{-1} \ln T$ & \cite{huang2016following,molinaro2020} \\
			$\|\G_t\|\geq L t, \forall t$   & & & &&    \\
			\hline
			General	  & $\lambda$-smooth  & Membership/Separation & 1   &  $(\lambda +1)\sqrt{T}$ & \cite{levy2019}  \\
			General	  & General  & Membership/Separation & $1$ to $\ln  T$  &  $\kappa \sqrt{T}$ & \cite{mhammedi2021efficient,garber2022new}  \\
			\hline
		\end{tabular}
		\caption{Summary of new (this paper) and existing regret bounds for different OCO settings. In the rows containing $\kappa$, it is assumed that $\K$ satisfies $\mathbb{B}(r)\subseteq \K\subseteq  \mathbb{B}(R)$ for some $r,R>0$ and $\kappa = R/r$. In the antepenultimate row, $\G_t \coloneqq \sum_{s=1}^t\g_s$ denotes the sum of the observed subgradients $(\g_t)$.} 
		\label{tab:result}
	\end{table}
	
	\paragraph{Outline.} 
	In \S\ref{sec:prelim}, we introduce our notation and necessary definitions. In \S\ref{sec:builder}, we present some background results in OCO we need to develop our algorithms. In \S\ref{sec:projectionfree}, we present our main algorithm for strongly convex sets, and in \S\ref{sec:general}, we extend our method to general sets by approximating them with strongly convex sets and designing efficient LOOs for the latter. We conclude with a discussion in \S\ref{sec:discussion}.
	
	\section{Notation and Definitions}
	\label{sec:prelim}
	In this section, we present our notation and necessary definitions for the remainder of the paper. Throughout, we will assume that $\K$ is a closed convex set containing the origin and such that $\K \subseteq \mathbb{B}(R)$, where $\mathbb{B}(R)$ denotes the Euclidean ball of radius $R>0$. 
	For a convex function $f\colon \reals^d \rightarrow \reals\cup \{+\infty\}$, we let $\partial f(\x)\coloneqq \{\g\in \reals^d \colon f(\y)\geq f(\x) + \inner{\g}{\y -\x}, \forall \y \in \reals^d \}$ denote the subdifferential of $f$ at $\x\in \reals^d$. 
	As is standard in OCO \cite{cesa2006prediction}, we will express our bounds in terms of the \emph{linearized regret} $\sum_{t=1}^T \inner{\g_t}{\w_t - \u}$ where $(\w_t)$ are the iterates of the algorithm; $(\g_t \in \partial f_t(\w_t))$; $(f_t)$ the sequence of convex losses chosen by the environment; and $\u$ is a comparator in $\K$. By convexity, we have $f_t(\w_t)- f_t(\u)\leq \inner{\g_t}{\w_t-\u}$, for all $\u$, and so a bound on the linearized regret immediately implies a bound on the standard regret. We let $\|\cdot\|$ denote the Euclidean norm in $\reals^d$. For any $f\colon \cK\rightarrow \reals$, we let $\argmin_{\u \in \cK} f(\u)$ denote the subset of points in $\cK$ that minimize $f$. For any $a,b \in \mathbb{N}$ s.t.~$a\leq b$, we let $[a..b]\coloneqq \{a,a+1,\dots, b\}$, $[b]\coloneqq [1..b]$, and $x_{a..b}\coloneqq (x_a,\dots,x_b)$, for any sequence $(x_i)$. We use $\wtilde{O}(\cdot)$ to hide poly-log-factors in problem parameters. We now present necessary definitions in convex analysis (see e.g.~\cite{hiriart2004}), where we assume that $\cK\subseteq \reals^d$ is a convex set.  
	\begin{definition}
		\label{def:supportfun}
		The {\em support function} $\sigma_{\cK}$ of $\cK$ is $\sigma_{\cK}(\w)\coloneqq \sup_{\u\in \cK} \inner{\u}{\w}$, for $\w\in \reals^d$. The {\em Gauge function} $\gamma_{\cK}$ of $\cK$ is defined as $\gamma_{\cK}(\w)\coloneqq \inf\{\lambda>0:\w \in \lambda \cK\}$, for $\w\in \reals^d$.
	\end{definition}
	\begin{definition}
		\label{def:polar}
		The \emph{polar set} $\cK^\circ$ of $\cK$ is defined as $\cK^\circ \coloneqq \{\u \in \reals^d\colon \inner{\u}{\w}\leq1, \ \forall \w\in \cK\}$.
	\end{definition}
	
	\begin{definition}[\cite{garber2015faster}] The set $\cK$ is $\mu$-strongly convex w.r.t.~a norm $\|\cdot \|$, for $\mu>0$, if for any $\theta \in[0,1]$ and any $\x, \y \in \cK$ and $\v\in \reals^d$ s.t.~$\|\v\|\leq \mu \theta (1-\theta)\|\x- \y\|^2/2$, we have $\theta \x + (1-\theta)\y + \v \in \cK$.
	\end{definition}
	Strongly convex sets include, for example, the balls induced by $\ell_p$ norms, shatten norms, and group norms \cite{garber2015faster}. A partial characterization of strongly convex sets is given in \cite[Lemma 2]{garber2015faster}. It is known that strong convexity of a set $\cK$ implies a form of Lipschitz-continuity of the support function $\sigma_{\cK}$ \cite{vial1982strong, polovinkin1996strongly, abernethy2018faster}. We now state this result (see \cite[Lemma 18]{abernethy2018faster}), which will be key in our analysis:
	\begin{lemma}
		\label{lem:strong}
		If $\cK$ is $\mu$-strongly convex for $\mu>0$ w.r.t.~$\|\cdot\|$, then for any $\x,\y \in \reals^d\setminus \{\bm{0}\}$, we have 
		\begin{align}
			\forall (\u,\v) \in \partial \sigma_{\cK}(\x)\times \partial \sigma_{\cK}(\y), \qquad 	\|\u  - \v\| \leq {2 \|\x- \y\|}{(\mu (\|\x\|+ \|\y\|))^{-1}}.
		\end{align}
	\end{lemma}
	\section{Building Blocks: Helper OCO Algorithms}
	\label{sec:builder}
	In this section, we present some OCO algorithms that will constitute the building blocks of our main algorithm. In particular, we present the classical Follow-The-Leader (FTL) algorithm, which is projection-free, and state its regret guarantee on strongly convex sets. We then present the comparator-adaptive algorithm $\mathsf{FreeGrad}$ \cite{mhammedi2020}, whose regret bound scales with the magnitude of the comparator. Finally, we will combine FTL and $\mathsf{FreeGrad}$ into a single algorithm that we call \emph{Follow-The-Scaled-Leader} (FTSL). FTSL inherits the projection-free and comparator-adaptive properties of FTL and $\mathsf{FreeGrad}$, respectively. FTSL is the main building block of our projection-free algorithm (introduced in \S\ref{sec:projectionfree}) that we show guarantees the optimal (up to log-factors) $\wtilde {O}(\sqrt{T})$ regret bound on strongly convex sets (the comparator-adaptive property of FTSL is key here). 
	\paragraph{Follow-The-Leader.} To build our projection-free algorithm for strongly convex sets, we will use the Follow-The-Leader (FTL) subroutine displayed in Algorithm \ref{alg:boundarywolf}. It is clear that FTL requires only one Linear Optimization step per round and no projections onto $\K$. When the set $\K$ is strongly convex, one can use the `Lipschitz-continuity' property of the support function in Lemma \ref{lem:strong} and the standard regret decomposition of FTL to derive the following regret bound for FTL (see e.g.~\cite{molinaro2020} for a similar statement and App.~\ref{app:FTLregret} for a proof): 
	\begin{lemma}[Regret of FTL on $\K$]
		\label{lem:FTLregret}
		Suppose that $\K$ is a $\mu$-strongly convex set w.r.t.~$\|\cdot\|$ for $\mu>0$. Let $(\w_t)$ be the iterates of $\cA_{\FTL}(\K)$ (Alg.~\ref{alg:boundarywolf}) in response to any sequence of  subgradients $(\g_s \in \partial f_s(\w_s))$. If $(\g_s)$ satisfy $\|\sum_{s=1}^t\g_s\|>0$, for all $t\in[T]$, then for all $T\geq 1$ and all $\u \in \K$,
		\begin{align}
			\sum_{t=1}^T\inner{\g_t}{\w_t - \u} \leq  	\sum_{t=1}^T \frac{2\|\g_t\|^2}{\mu \|\sum_{s=1}^t \g_s\|}. \label{eq:FTLreg}
		\end{align} 
	\end{lemma}
	\paragraph{$\mathsf{FreeGrad}$.} Another building block in our main algorithm is a comparator-adaptive subroutine $\cA_{\mathsf{FG}}$ in 1d, whose regret against any comparator $u\in \reals$ is less than $\wtilde O(|u| \sqrt{T})$ after $T$ rounds. Though there are many recent algorithms that achieve such a guarantee, we choose $\cA_{\mathsf{FG}}$ to be $\mathsf{FreeGrad}$ \cite{mhammedi2020} since I) its outputs have a closed-form expression; and II) it has a data-dependent regret bound that can be much smaller than $\wtilde O(|u| \sqrt{T})$ in practice. We first describe $\mathsf{FreeGrad}$ in $\reals^d$. If we let $\G_{t-1}\coloneqq \sum_{s=1}^{t-1}\g_s$ [resp.~$V_{t-1}\coloneqq \sum_{s=1}^{t-1}\|\g_s\|^2$] be the sum of past subgradients [resp.~squared subgradient norms], then the output of $\mathsf{FreeGrad}$ at round $t$ is given by $\w_t = - \G_{t-1}\cdot \Psi_L(\|\G_{t-1}\|, V_{t-1})$, where 
	\begin{align}
		\Psi_L(s,v)\coloneqq    \frac{ (2v +L |s|)\cdot L^2 }{2(v+ L|s |)^2 \ \sqrt{L v}} \cdot \exp\left(\frac{|s|^2}{2 v + 2L |s|} \right),\label{eq:leanrning}
	\end{align}
	and $L>0$ is a scale parameter for the norm of the subgradients. The quantity $\Psi_L(\|\G_{t-1}\|, V_{t-1})$ can be seen as an adaptive learning rate. The instantiation of $\mathsf{FreeGrad}$ in one dimension, which we denote by $\cA_{\od}$ from now on, is summarized in Alg.~\ref{alg:freegrad} and has the following guarantee \cite[Thm.~6]{mhammedi2020}:  
	\begin{lemma}
		\label{lem:freegradreg}
		Let $L>0$ and $(z_t)$ be the outputs of $\cA_{\od}(L)$ in response to any sequence of subgradients $(g_t \in \partial f_t(z_t))$ s.t.~$| g_t|\leq L$. Then, for all $z\in \reals$, $V_T \coloneqq \sum_{t=1}^T g_t^2$, and $R_T(z)\coloneqq \sum_{t=1}^T \inner{g_t}{z_t-z}$\emph{:}
		\begin{align}
			R_T(z) \leq  2 |z| \sqrt{V_T \ln(1+2|z|V_T/L^2)} + 4 L |z| \ln(4  |z|\sqrt{V_T}/L)+ L.
		\end{align}
	\end{lemma}
	\noindent A crucial property of this bound is that the regret against the origin is a constant.
	
\hspace{-0.6cm}	\begin{minipage}{.43\linewidth}
		\begin{algorithm}[H]
			\caption{Follow-The-Leader $\cA_{\FTL}(\K)$}
			\label{alg:boundarywolf}
			\begin{algorithmic}
				\REQUIRE Convex set $\K$.
				\STATE Set $\w_1 \in \K$. 
				\FOR{$t=1,2,\dots$}
				\STATE Play $\w_t$ and observe $\g_t \in \partial f_t(\w_t)$.
				\STATE Set $\w_{t+1} \in \argmin_{\w\in \K} \sum_{s=1}^{t} \w^\top \g_s$.
				\ENDFOR
			\end{algorithmic}
		\end{algorithm}
	\end{minipage}
	\begin{minipage}{.55\linewidth}
		\begin{algorithm}[H]
			\caption{1d $\mathsf{FreeGrad}$ w.~param.~$L>0$. $\cA_{\od}(L)$}
			\label{alg:freegrad}
			\begin{algorithmic}
				\STATE Set $z_1=G_0=V_0=0$. 
				\FOR{$t=1,2,\dots$}
				\STATE Play $z_t$ and observe $g_t \in  \partial f_t(z_t)$. 
				\STATE Set $G_t=G_{t-1}+g_t$ and $V_t = V_{t-1}+g_t^2$.
				\STATE Set $z_{t+1}= - G_{t}\cdot  \Psi_L(G_{t},V_{t})$, with $\Psi_L$ as in \eqref{eq:leanrning}.
				\vspace{0.05cm}
				\ENDFOR
			\end{algorithmic}
		\end{algorithm}
	\end{minipage}
	
	\paragraph{Comparator Adaptive Algorithm Wrapper.} We now present an algorithm wrapper $\cA_{\mathsf{CA}}$ (Alg.~\ref{alg:comparatoradaptive}) that takes in any base algorithm $\cB$ defined on $\K$ and uses $\cA_{\od}$ to essentially turn $\cB$ into a comparator adaptive algorithm. $\cA_{\mathsf{CA}}$ is based on a recent $\reals^d$-to-$\reals$ OCO reduction by \cite[Section 3]{cutkosky2018black}.
	In the next section, where we consider a strongly convex set $\K$, we will set the subroutine $\cB$ to FTL. We call the resulting algorithm Follow-The-Scaled-Leader, and we write 
	\begin{align}
		\cA_{\FTSL}(\cdot)\coloneqq  \cA_{\mathsf{CA}}(\cA_{\mathsf{FTL}}(\K),\cdot). \label{eq:FTSL}
	\end{align}Like FTL, $\cA_{\FTSL}$ will be projection-free. But unlike FTL, the regret of $\cA_{\FTSL}$ at the origin will be a constant---regardless of the magnitude of the sum of the subgradients. Leveraging this property will be key in showing that our main algorithm, which uses $\cA_{\FTSL}$ as a subroutine, guarantees a $\wtilde O(\sqrt{T})$ regret bound on strongly convex sets. We now state a guarantee for Alg.~\ref{alg:comparatoradaptive} whose proof follows simply by adding and subtracting $\sum_{t=1}^T  \gamma_{\K}(\u) \g_t^\top \w_t$ from the regret of Alg.~\ref{alg:comparatoradaptive}, where $\gamma_{\K}$ is as in Definition~\ref{def:supportfun}.
	\begin{lemma}
		\label{lem:parameterfree-frankwolfe}
		Let $L>0$. Further, let $\cA_{\od}\equiv \cA_{\od}(L)$ and $\cB$ be the subroutines within Algorithm \ref{alg:comparatoradaptive}. Then, the outputs $(\u_t)$ of Alg.~\ref{alg:comparatoradaptive} in response to any sequence of subgradients $(\g_t\in \partial f_t(\u_t))$, satisfy	
		\begin{align}
			\forall \u \in \K,\quad	\sum_{t=1}^T \inner{\g_t}{\u_t - \u} = R^{\cA_{\od}}_T(\gamma_\K(\u))  + \gamma_\K(\u) \cdot  R^{\cB}_T(\u/\gamma_{\K}(\u)), \label{eq:subtract}
		\end{align}
		where $R^{\cA_{\od}}_T(\cdot)\coloneqq \sum_{t=1}^T  (z_t - \cdot) \inner{\w_t}{\g_t}$ \emph{[}resp.~$R^{\cB}_T(\cdot)\coloneqq \sum_{t=1}^T   \inner{\g_t}{\w_t-\cdot}$\emph{]} is the regret of $\cA_{\od}$ \emph{[}resp.~$\cB$\emph{]} in response to the losses $(z \mapsto \inner{\w_t}{\g_t}z)$ \emph{[}resp.~$\w\mapsto \inner{\w}{\g_t}$\emph{]}.
	\end{lemma}
	\begin{algorithm}[t]
		\caption{Comparator Adaptive Algorithm Wrapper $\cA_{\mathsf{CA}}(\cB, 
			L)$.  {\color{blue}($\cA_{\FTSL}(\cdot)=\cA_{\mathsf{CA}}(\cA_{\FTL}(\K),\cdot)$)}}
		\label{alg:comparatoradaptive}
		\begin{algorithmic}[1]
			\REQUIRE OCO subroutine $\cB$ on $\K$. Parameter $L>0$. Instance $\cA_{\od}\equiv \cA_{\od}(L)$ of Alg.~\ref{alg:freegrad}.
			\FOR{$t=1,2,\dots$}
			\STATE Get $z_t$ and $\w_t\in \K$ from $\cA_{\od}$ and $\cB$, respectively.
			\STATE Play $\u_t = z_t \w_t\in \K$ and observe $\g_t \in \partial f_t(\u_t)$.
			\STATE  Send $z \mapsto \inner{\w_t}{\g_t}z$ and $\w\mapsto \inner{\w}{\g_t}$ as the $t^{\text{th}}$ loss function to $\cA_{\od}$ and $\cB$, respectively. 
			\ENDFOR
		\end{algorithmic}
	\end{algorithm}
	\section{A Projection-Free Algorithm for Strongly Convex Sets}
	\label{sec:projectionfree}
	In this section, we start by describing the ideas leading up to our main algorithm (displayed in Alg.~\ref{alg:EffAlg}). Along the way, we illustrate how using the $\cA_{\FTSL}$ subroutine defined in \eqref{eq:FTSL} as a building block allows us to achieve the desired $\wtilde {O}(\sqrt{T})$ regret bound for strongly convex sets. We state the guarantee of Alg.~\ref{alg:EffAlg} in Thm.~\ref{thm:master}, which we follow by a proof sketch. The full proof is presented in App.~\ref{sec:proofmain}.   
	We now discuss the key ideas leading to our solution. In what follows, $R>0$ is such that $\K\subseteq \mathbb{B}(R)$.

	\paragraph{Warmup: an impracticable algorithm.} Let $\tau \in[T]$ and consider an Algorithm $\cA$ that outputs $\w_{1..T}\subset \K$, where $\w_{1..\tau}$ [resp.~$\w_{\tau+1..T}$] are the iterates of an instance of $\cA_{\FTSL}$ (see \eqref{eq:FTSL}) that is initialized at round $1$ [resp.~round $\tau+1$]. Suppose $\tau$ is the last round such that $\left\|\G_\tau\right\|^2 \leq \frac{V_\tau}{\mu R}$, where $\G_\tau \coloneqq \sum_{s=1}^\tau \g_s$, $V_\tau \coloneqq\sum_{s=1}^\tau \|\g_s\|^2$, and $(\g_s \in \partial f_s(\w_s))$ are the observed subgradients. That is,
	\begin{align}
		\tau = \max \{t \in[0..T]\colon \  \left\|\G_t\right\|  \leq  \sqrt{ V_t/(R \mu)} \}.\label{eq:definition}
	\end{align}
	Under this assumption, we can bound the regret of $\cA$ between the rounds $1$ and $\tau$ as follows: $\forall \u\in \K$, 
	\begin{align}
		\sum_{s=1}^\tau \inner{\g_s}{\w_s - \u}\leq \sum_{s=1}^\tau \inner{\g_s}{\w_s}  +  \|\u\| \cdot \|\G_\tau\|  \leq L+  \sqrt{ RV_\tau/\mu}  = O (L\sqrt{R T/\mu}), \label{eq:thankscomp}
	\end{align}
	where the last inequality follows by the definition of $\tau$ and the fact that the regret at the origin of $\cA_{\FTSL}$, i.e.~$\sum_{s=1}^\tau  \inner{\g_s}{\w_s}$, is bounded by $L$ (by Lemma \ref{lem:parameterfree-frankwolfe} and the fact that $\gamma_{\K}(\bm{0})=0$). Note that had we used $\cA_{\FTL}$ instead $\cA_{\FTSL}$, we would not have been able to guarantee a non-vacuous bound on the regret when $\tau=T$. This is because, without a lower bound on the norm of the sum of the observed subgradients (which is the case when $\tau=T$), the regret bound of FTL in Lemma \ref{lem:FTLregret} can be vacuous (this is true for any comparator including the origin). $\cA_{\FTSL}$ bypasses this limitation by ensuring its regret at the origin is always bounded by a constant. We also recall that $\cA_{\FTSL}$ is projection-free, which is why we use it instead of alternative OCO algorithms.

	We now consider the regret bound of $\cA$ on the interval $[\tau+1..T]$. Recall that $\w_{\tau+1..T}$ are the outputs of an instance of $\cA_{\FTSL}$ that is initialized at round $\tau +1$. Thus, using that the regret of $\cA_{\FTSL}$ against $ \u \in \K$ is equal to $R^{\cA_{\od}}_T(\gamma_\K(\u))  + \gamma_\K(\u)  R^{\cA_{\FTL}}_T({\u}/{\gamma_{\K}(\u)}),$ (Lem.~\ref{lem:parameterfree-frankwolfe}) and Lem.~\ref{lem:FTLregret}, we get $\forall \u \in \K$,
	\begin{align}
		\sum_{s=\tau+1}^{T} \inner{\g_s}{\w_s-\u}	\leq R_T^{\cA_{\od}}(\gamma_{\K}(\u))+ \sum_{t=\tau+1}^T \frac{2 \gamma_{\K}(\u)\|\g_t\|^2}{\mu \|\sum_{s=\tau+1}^t \g_s\|} \stackrel{(*)}{\leq} \wtilde O \left(\sqrt{\frac{R  T}{\mu\wedge R^{-1}}}\right),\label{eq:secondpart}
	\end{align}
	where $(*)$ follows by the fact that I) $\gamma_{\K}(\u)\leq 1$ (since $\u \in \K$); II) $R^{\cA_{\od}}(\gamma_{\K}(\u)) \leq \wtilde{O}(\sqrt{V_T})$ (Lem.~\ref{lem:freegradreg}); and III) that $\|\G_t\|\geq \sqrt{ V_t/(R\mu)}$ for all $t>\tau$ (a more detailed justification of $(*)$ will be given below in the proof sketch of Thm.~\ref{thm:master}). By combining \eqref{eq:thankscomp} and \eqref{eq:secondpart}, we get a $\wtilde{O}(\sqrt{T})$ regret bound for $\cA$. 
	
	\paragraph{A practicable algorithm using sleeping experts.} The above analysis for the regret of Algorithm $\cA$ relies crucially on $\tau$ satisfying \eqref{eq:definition}. In the adversarial setting of this paper, where the environment can change the sequence of loss functions depending on the learner's current and past outputs, it is not possible for the learner to know $\tau$ in advance. One way around this challenge is to seek an algorithm $\cA$ whose outputs $(\x_t)$ guarantee, 
	\begin{align}
		\forall \text{ intervals} \ I \subseteq [T],\ \forall \u \in \K, \quad 	\sum_{t\in I} \inner{\g_t}{\x_t - \u} \leq \sum_{t\in I } \inner{\g_t}{\w^{(\min I)}_t - \u}+ O(\sqrt{T}),  \label{eq:interval}
	\end{align}
	where $(\g_t\in \partial f_t(\x_t))$ and $(\w_t^{(s)})$ are the outputs of an instance of $\cA_{\FTSL}$ initialized on round $s\leq t$ in response to the sequence of losses $(\w \mapsto \inner{\w}{\g_t})$. In words, \eqref{eq:interval} means that for any interval $I =[i..j]$, Alg.~$\cA$ performs well on $I$ relative to an instance of $\cA_{\FTSL}$ initialized at round $i$.  Requiring this for any $i,j\in[T]$, means we can instantiate \eqref{eq:interval} with $I = [1..\tau]$ and $I = [\tau+1..T]$, with $\tau$ as in \eqref{eq:definition}, and invoke the regret bounds in \eqref{eq:thankscomp} and \eqref{eq:secondpart}, which would now hold for $(\w_t= \w_t^{(1)})$ and $(\w_t= \w_t^{(\tau+1)})$, respectively. This would imply a $\wtilde{O}(\sqrt{R\mu^{-1} T})$ regret bound for $\cA$. We stress that the round $\tau$ in \eqref{eq:definition} need not be known by $\cA$ in this case, thanks to the `adaptive' guarantee of $\cA$ in \eqref{eq:interval}.

	It remains to build an Algorithm $\cA$ that satisfies \eqref{eq:interval}. By treating instances of $\cA_{\FTSL}$ initialized at different rounds as `experts', there exist algorithms in the setting of Prediction with Expert Advice (PEA) that can be used to build an algorithm $\cA$ that satisfies \eqref{eq:interval}. The PEA setting (in particular the sleeping experts' setting) and the algorithmic ideas we borrow from it are detailed in App.~\ref{sec:proofmain}. The only challenge remaining is computational in nature, since this approach would require keeping track of at least $T$ experts; instances of $\cA_{\FTSL}$ initialized at different rounds. Since each instance of $\cA_{\FTSL}$ requires one call to an LOO per round, this would result in a total of $\Omega(T^2)$ calls to the Oracle after $T$ rounds. To circumvent this problem, we design a tailored sleeping experts' algorithm for our setting that requires only \emph{two} active experts per round, and thus makes only two LOO calls per round. 
	
	We now state the guarantee of our main algorithm (Alg.~\ref{alg:EffAlg}) that puts all the above ideas together. 
	\begin{algorithm}
		\caption{Projection-free OCO using Sleeping Experts and comparator adaptive algorithms.}
		\label{alg:EffAlg}
		\begin{algorithmic}[1]
			\REQUIRE OCO subroutine $\cB$ on $\K$. Parameters $L,\eta,\mu>0$. 
			Base algorithms $(\cA_s)$, where $\cA_s$ is an instance of $\cA_{\mathsf{CA}}(\cB, L)$ (Alg.~\ref{alg:comparatoradaptive}) that is initialized at the beginning of round $t=s$.
			\vspace{0.1cm}
			\STATE Set $\tau=1$, $q_1 =1/2$, and $U_{0}=W_{0}=S_0=0$. 
			\FOR{$t=1,2,\dots$}
			\STATE Get $\u_{t}$ and $\w_{t}$ from $\cA_1$ and $\cA_\tau$, respectively. \algcomment{$\w_t$ is $\cA_\tau$'s $(t-\tau+1)^{\text{th}}$ output} 
			\STATE Play $\x_t = q_t \u_{t} +(1-q_t) \w_{t} \in \K$ and observe $\g_t \in \partial f_t(\x_t)$ and set $\tilde \g_t = \g_t  \mathbf{1}\{\|\g_t\|\geq \frac{L}{t}\}$.
			\STATE Send $\w\mapsto \inner{\w}{\tilde \g_t}$ to $\cA_1$ [resp.~$\cA_\tau$] as the $t^\text{th}$ [resp.~$(t - \tau +1)^{\text{th}}$] loss function.
			\STATE Set $U_{t} = U_{t-1}+\inner{\tilde \g_t}{\u_{t}}$, $W_{t} = W_{t-1}+\inner{\tilde \g_t}{\w_{t}}$, and $S_t=S_{t-1}+\inner{\tilde \g_t}{\x_t}$.
			\IF{$ R\mu \|\sum_{s=1}^t \tilde \g_s\|^2 \leq  \sum_{s=1}^t \|\tilde \g_s\|^2 \cdot \ln T$}
			\STATE Set $\tau = t+1$ and $W_{t}=S_{t}$.
			\ENDIF
			\STATE Set $q_{t+1}=e^{-\eta  U_{t}}/(e^{-\eta  U_{t}}+ e^{-\eta  W_{t}})$.
			\ENDFOR
		\end{algorithmic}
	\end{algorithm}
	\begin{theorem}
		\label{thm:master}
		Let $(\x_t)$ be the outputs of Alg.~\ref{alg:EffAlg} in response to any sequence of subgradients $(\g_t\in \partial f_t(\x_t))$. If $\K$ is a $\mu$-strongly convex set w.r.t.~$\|\cdot\|$ and Alg.~\ref{alg:EffAlg} is run with $\cB=\cA_{\FTL}(\K)$ (Alg.~\ref{alg:boundarywolf}) and parameters $\eta,\mu, L>0$ such that $\max_{t\in[t]}\|\g_t\|\leq L$ and $2R L\eta \leq 1$, then for $V_T \coloneqq \sum_{t=1}^T \|\g_t\|^2$,
		\begin{align}
			\sum_{t=1}^T \inner{\g_t}{\x_t - \u} \leq \frac{\ln (T+1)}{\eta} + \frac{3 \eta R^2}{4}V_T  +\wtilde O(\sqrt{{R}V_T/\mu}), \quad \forall \u \in \K. \label{eq:optimizaedbound}
		\end{align}
	\end{theorem}
	\paragraph{Parameter-Turning.} Setting $\eta = \sqrt{\ln T}/(RL\sqrt{T})$ in Theorem \ref{thm:master} leads to the first OCO algorithm that makes two calls per round to an LOO and guarantees a $\wtilde{O}(\sqrt{T})$ regret bound for strongly convex sets and general convex functions. We note that it is possible to tune $\eta$ adaptively using, for example, the doubling trick (see e.g.~\cite{cesa2006prediction}) to achieve a data-dependent regret bound of the form: \begin{align}\forall \u \in \K, \quad \sum_{t=1}^T \inner{\g_t}{\x_t - \u} \leq O\left(\sqrt{R \tilde\mu^{-1} V_T \ln T} \right), \quad \text{where \ \ $\tilde\mu\coloneqq \min (\mu,1/R)$} \label{eq:regret}
	\end{align}
	and $V_T \coloneqq \sum_{t=1}^T \|\g_t\|^2$. Since this bound scales with $O(\sqrt{V_T})$ it can be much smaller than the standard worst-case $O(\sqrt{T})$ regret; especially when the loss functions are smooth \cite{srebro2010,cutkosky2018black}. Furthermore, via online-to-batch conversion,  \eqref{eq:regret} leads to the fast $O({1}/{T})$ rate in offline smooth optimization \cite{cutkosky2019}.

	\begin{proof}[{\bf Proof Sketch of Theorem \ref{thm:master}}]
		Let $\cT \coloneqq \{t\in[T] \colon R\mu \|\sum_{s=1}^{t-1} \tilde\g_s\|^2\leq  Q_{t-1} \ln T\}$, where $Q_{t}\coloneqq \sum_{s=1}^{t} \|\tilde \g_s\|^2$, for $t\in[T]$. Set $\bar \tau = \max{\cT}$ (note that $\cT\neq \emptyset$ since $1\in \cT$). First note that to get the desired result, it suffices to bound the `pseudo' regret $\sum_{t=1}^T \inner{\tilde \g_t}{\x_t - \u}$, where $\tilde \g_t \coloneqq \g_t  \mathbf{1}\{\|\g_t\|\geq \frac{L}{t}\}$. In fact, by Cauchy Schwarz, we have, for all $\u \in \K$ and $B_T \coloneqq 2R L \ln T$,
		\begin{align}
			\sum_{t=1}^T \inner{\g_t}{\x_t - \u} &\leq \sum_{t=1}^T (\inner{\tilde \g_t}{\x_t - \u} +   \mathbb{I}{\{\|\g_t\|< \tfrac{L}{t}\}}\inner{\g_t}{\x_t - \u}) \leq \sum_{t=1}^T\inner{\tilde \g_t}{\x_t - \u}  + B_T. \label{eq:reduction}
		\end{align}
		Thus, we now focus on bounding the `pseudo' regret. By our choice of weights $(q_t)$ and a result from the Sleeping Experts' setting, we show (see full proof in App.~\ref{sec:proofmain}) that the pseudo regret of Alg.~\ref{alg:EffAlg} is comparable to that of the subroutine $\cA_1$ [resp.~$\cA_{\bar \tau}$] on $[1..\bar \tau -1]$ [resp.~$[\bar \tau..T]$]. In particular,
		\begin{subequations}
			\begin{align}
				\sum_{t=1}^{\bar \tau -1} \inner{\tilde \g_t }{\x_t -\u} -\frac{\ln (T+1)}{\eta} - \frac{3\eta R^2}{4}Q_T  &\leq   \sum_{t=1}^{\bar \tau -1}  \inner{\tilde \g_t }{\u_t - \u}, \quad \text{(RHS = regret of $\cA_{1}$)} \label{eq:firstb}\\
				\sum_{t=\bar \tau}^{T}\inner{\tilde \g_t }{\x_t -\u} -\frac{\ln (T+1)}{\eta} - \frac{3\eta R^2}{4}Q_T 	& \leq	\sum_{t=\bar \tau}^{T} \inner{\tilde \g_t }{\w_t - \u},\quad \text{(RHS = regret of $\cA_{\bar\tau}$)} \label{eq:secondb}
			\end{align}
		\end{subequations}
		Note that $Q_T \leq V_T=\sum_{t=1}^T \|\g_t\|^2$. Thus, to get the desired regret bound, it suffices to bound the right-hand-sides of \eqref{eq:firstb} and \eqref{eq:secondb} by $O( \sqrt{R\mu^{-1} Q_T \ln T})$. We start with \eqref{eq:firstb}. Since $(\u_t)$ are the outputs of $\cA_1$ (the instance of Alg.~\ref{alg:comparatoradaptive} initialized at the beginning of round 1). Therefore, if we let $R^{\cA_1}_{1:\bar \tau -1}(\bm{0})$ be the regret of $\cA_1$ against the comparator $\bm{0}$ on the interval $[1..\bar \tau -1]$, we have
		\begin{align}
			\sum_{t=1}^{\bar \tau -1}  \inner{\tilde \g_t }{\u_t - \u} &= R^{\cA_1}_{1:\bar \tau -1}(\bm{0}) - \u^\top \sum_{t=1}^{\bar \tau-1}\tilde \g_t  \leq L + \sqrt{R\mu^{-1} Q_{T} \ln T},\label{eq:sanction}
		\end{align}
		where the last inequality follows by the fact that $R^{\cA_1}_{1:\bar \tau -1}(\bm{0}) \leq L$ (Lem.~\ref{lem:parameterfree-frankwolfe} and $\gamma_{\K}(\bm{0})=0$) and that $|\u^\top \sum_{t=1}^{\bar \tau-1}\tilde \g_t |\leq R \|\sum_{t=1}^{\bar \tau -1}\tilde \g_t \|\leq \sqrt{R\mu^{-1} Q_{T} \ln T}$ (by definition of $\bar \tau$). It remains to bound the RHS of \eqref{eq:secondb}. Note that by definition of $\bar \tau$, $(\w_t)_{t\in[\bar \tau..T]}$ are the outputs of $\cA_{\bar \tau}$ (the instance of Alg.~\ref{alg:comparatoradaptive} initialized at the beginning of round $\bar \tau$). Therefore, the RHS of \eqref{eq:secondb} is the regret $R^{\cA_{\bar \tau}}_{\bar \tau:T}(\bm{u})$ of $\cA_{\bar\tau}$ against comparator $\u$ between the rounds $[\bar \tau..T]$. Using Lem.~\ref{lem:parameterfree-frankwolfe} and \ref{lem:FTLregret}, the regret $R^{\cA_{\bar \tau}}_{\bar \tau:T}(\bm{u})$ satisfies 
		\begin{align}
			R^{\cA_{\bar\tau}}_{\bar \tau:T}(\u)   \leq R^{\cA_{\od}}_T(\gamma_{\K}(\u)) +  \sum_{t=\bar \tau}^T \frac{2\gamma_{\K}(\u)\|\tilde \g_t \|^2}{\mu\|\sum_{s=\bar \tau}^t \tilde \g_s\|}
			\leq O(R \sqrt{Q_T \ln T}) +\sum_{t=\bar \tau}^T \frac{2\|\tilde \g_t \|^2}{\mu\|\sum_{s=\bar \tau}^t \tilde \g_s\|}, \label{eq:pluto10}
		\end{align}
		for all $\u\in\K$, where we used that $\gamma_{\K}(\u)\leq 1$ for $\u \in \K$. We now bound the last term in \eqref{eq:pluto10}. By definition of $\bar \tau$ and the triangle inequality, we have 
		\begin{align}
			\forall t\geq \bar \tau, \quad \left\|\sum_{s=\bar \tau}^t \tilde \g_s \right\| &\geq 	\left\|\sum_{s=1}^t \tilde \g_s \right\|  - 	\left\|\sum_{s=1}^{\bar \tau -1} \tilde \g_s \right\| \geq \sqrt{R^{-1}\mu^{-1} Q_{1:t} \ln T} - \sqrt{R^{-1}\mu^{-1} Q_{1:\bar \tau-1}\ln T},
		\end{align}
		where $Q_{i:j} \coloneqq \sum_{s=i}^j \|\tilde \g_s\|^2$. Using this, we can bound the last sum in \eqref{eq:pluto10} by
		\begin{align}
			\sum_{t=\bar \tau}^T \frac{\|\tilde \g_t \|^2}{\|\sum_{s=\bar \tau}^t\tilde \g_s\|}  \leq \sum_{t=\bar \tau}^T \frac{ \sqrt{{R\mu}(\ln T)^{-1}} \|\tilde \g_t \|^2}{\sqrt{Q_{1:t}}- \sqrt{Q_{1:\bar \tau-1}}} 
			& = \sqrt{\frac{R\mu}{\ln T}}\sum_{t=\bar \tau}^T \frac{\|\tilde \g_t \|^2( \sqrt{Q_{1:t}}+ \sqrt{Q_{1:\bar \tau -1}})}{Q_{\bar \tau:t}},\nn \\
			& \leq 2\sqrt{\frac{R\mu Q_{1:\bar \tau -1}}{\ln T}} \sum_{t=\bar \tau}^T \frac{\|\tilde \g_t \|^2}{Q_{\bar \tau:t}} +  \sqrt{\frac{R\mu }{\ln T}} \sum_{t=\bar \tau}^T \frac{\|\tilde \g_t \|^2}{\sqrt{Q_{\bar \tau:t}}},\nn \\
			& \leq 4\sqrt{\frac{R\mu Q_{T} }{\ln T}}\ln \frac{Q_{\bar \tau:T}}{\|\tilde \g_{\bar \tau}\|^2}+  2 \sqrt{\frac{R\mu Q_{\bar \tau:T}}{\ln T}}. \label{eq:fan}
		\end{align}
		(See \eqref{eq:just} in App.~\ref{sec:proofmain} for the last step.) Finally, note that by definition of $\bar \tau$, we have $\|\tilde\g_{\bar \tau}\|>0$. This in turn implies that $\|\tilde \g_t\|\geq L/t$ since $\tilde \g_t = \g_t \cdot \mathbf{1}\{\|\g_t\| \geq L/t\}$. Therefore, we have $Q_{\bar \tau:T}/\|\tilde \g_{\bar \tau}\|^2 \leq T$. Combining this with \eqref{eq:fan}, \eqref{eq:pluto10}, \eqref{eq:sanction}, \eqref{eq:firstb}, \eqref{eq:secondb}, and \eqref{eq:reduction}, we obtain the desired result.
	\end{proof}

	\section{An Algorithm for General Convex Sets}
	\label{sec:general}
	In this section, we show how our algorithm for strongly convex sets can be used in the general setting. The proofs are deferred to Appendix \ref{sec:proofs}. Without loss of generality (by applying a reparametrization if necessary, see e.g.~\cite[Sec.~6]{mhammedi2021efficient}), we assume that there exists $r,R>0$ such that the set $\K$ satisfies
	\begin{align}
		\mathbb{B}(r)\subseteq \K \subseteq \mathbb{B}(R). \label{eq:sandwitch}
	\end{align} 
	By approximating $\K$ with a strongly convex set we will be able to apply the algorithm from \S\ref{sec:projectionfree} and get a $\wtilde O(T^{2/3})$ regret guarantee. We show that this can be achieved with an $\wtilde O(dT)$ expected number of calls to an LO Oracle for $\K$ after $T$ rounds. To achieve this, we will approximate the feasible set $\K$ by the strongly convex set $\K_{\epsilon}$:
	\begin{align}
		\K_{\epsilon} \coloneqq (\sqrt{1-\epsilon}\K^\circ \oplus \sqrt{\epsilon} \mathbb{B}(r)^\circ)^\circ,\quad \text{( ${}^\circ$ denotes the polar as in Def.~\ref{def:polar})} \label{eq:approx}
	\end{align}
	where $\oplus$ is the $L_2$ addition of sets; for any two convex sets $A,B$, $A\oplus B$ is the unique closed convex set whose support function satisfies $\sigma_{A\oplus B}^2 = \sigma^2_A + \sigma^2_B$ \cite{firey1962p}. \cite{molinaro2020} showed that $\K_{\epsilon}$ approximates $\K$ well. Furthermore, $\K_{\epsilon}$ is strongly convex with respect to $\|\cdot\|$. We note that using the set $((1-\epsilon)\K^\circ + \epsilon \mathbb{B}(r)^\circ)^\circ$ (also strongly convex \cite{molinaro2020}) in lieu of $\K_\epsilon$ would lead to a suboptimal $O(T^{3/4})$ regret after tuning $\epsilon$.
	\begin{lemma}
		\label{lem:approxlemma}
		Suppose that $\K$ satisfies \eqref{eq:sandwitch} and let $\kappa \coloneqq R/r$. Then, for all $\epsilon \in[0,1]$, the convex set $\K_{\epsilon}$ in \eqref{eq:approx} is  {\bf I)} $2 \epsilon/(r \sqrt{1+\epsilon \kappa^2})$-strongly convex with respect to $\|\cdot\|$, and {\bf II)} $\K_\epsilon \subseteq \K \subseteq \sqrt{1 + \kappa^2 \epsilon} \cdot \K_\epsilon$.
	\end{lemma}
	The proof of the lemma is in App.~\ref{sec:approxlemma}. Since $\K_{\epsilon}$ is strongly convex and approximates $\K$ well, the idea now is to run our algorithm from the previous section (Alg.~\ref{alg:EffAlg}) on $\K_{\epsilon}$, which would require $\cA_{\FTL}(\K_\epsilon)$ as the subroutine $\cB$ in Alg.~\ref{alg:EffAlg}, and tune $\epsilon$ to obtain a $\wtilde O(T^{2/3})$ regret bound against comparators in $\K$. The only challenge remaining is the ability to perform Linear Optimization on $\K_\epsilon$ efficiently (to have an efficient implementation of the $\cA_{\FTL}(\K_\epsilon)$ subroutine).
	We note that \cite{molinaro2020} essentially showed that LO on $\K_\epsilon$ can be done using $\wtilde O(d^m)$ calls to an LOO $\cO_{\K}$ for $\K$, for some $m\in \mathbb{N}$. However, \cite{molinaro2020} did not provide any upper bound on $m$. 
	In this work, we take a different approach to \cite{molinaro2020} and show that LO on $\K_{\epsilon}$ can be done efficiently using only $\wtilde{O}(d)$ expected number of calls to $\cO_{\K}$. We do this by reducing LO on $\K_{\epsilon}$ to Euclidean projection onto $\K^\circ$ (Lem.~\ref{lem:optCeps} and Rem.~\ref{rem:projection} below). We then show that the latter can be done using $\wtilde{O}(d)$ expected number of calls to $\cO_{\K}$. The next lemma constitutes our main technical result for this section:
	\begin{lemma}
		\label{lem:optCeps}
		Let $\epsilon \in(0,1), \w\in \reals^d\setminus \{\bm{0}\}$ and define $\Psi(\gamma, \u)\coloneqq \gamma^2 +  r^2 \|\gamma \u  -\w\|^2$. For $(\gamma_*, \u_*)\in \argmin_{\gamma \geq 0, \u \in \sqrt{1-\epsilon}\K^\circ}\Psi(\gamma,\u)$, let $\lambda_* \coloneqq \sqrt{\Psi(\gamma_*,\u_*)}$ and $\z_*\coloneqq {\w}-\gamma_*\u_*$. Then,
		\begin{align}\inner{\z_*}{\w}>\frac{\|\w\|^2 \epsilon}{2 \kappa^3 (1+\kappa^2 \epsilon)^{3/2}}, \quad \text{and} \quad  \v_* \coloneqq \frac{\lambda_* \z_*}{\inner{\z_*}{\w}} \in \argmax_{\v\in \K_\epsilon}\inner{\v}{\w}.\label{eq:cold}
		\end{align}
	\end{lemma}
	The lemma, whose proof is in App.~\ref{sec:optCeps}, tells us that Linear Optimization on the approximate set $\K_\epsilon$ can be performed efficiently whenever optimizing the function $\Psi$ in Lemma \ref{lem:optCeps} can be done efficiently. The next lemma gives us a clue as to how we can perform the latter (the proof is in App.~\ref{sec:oned}):
	\begin{lemma}
		\label{lem:oned}
		Let $\epsilon,r>0, \w\in \reals^d$. Then, for all $b>0$, $\Theta\colon\gamma \mapsto \gamma^2 + r^2 \inf_{\u \in \sqrt{1-\epsilon}\K^\circ}\|\gamma \u - \w\|^2$ is a differentiable $1$-strongly convex function on $(0,b]$ and $|\Theta'(\gamma)|\leq 4b+2 r \|\w\|$, $\forall \gamma\in (0,b]$. 
	\end{lemma}
	
	\begin{remark}
		\label{rem:projection}
		Lem.~\ref{lem:oned} implies that as long as we can evaluate $S(\gamma, \w)\coloneqq \inf_{\u \in \sqrt{1-\epsilon}\K^\circ}\|\gamma\u - \w\|^2$ efficiently, which amounts to performing a Euclidean projection of $\w$ onto $\gamma\sqrt{1-\epsilon}\K^\circ$, minimizing the function $\Psi$ in Lem.~\ref{lem:optCeps} reduces to minimizing the convex function $\Theta$ in Lem.~\ref{lem:oned}. The latter can be done at a linear rate using a Golden-Section search thanks to the Lipschitzness of $\Theta$ (see Lem.~\ref{lem:oned}). 
	\end{remark}
	We now show that approximate Euclidean Projections onto the set $\sqrt{1-\epsilon}\K^\circ$, which are needed to compute the function $S$ in Remark \ref{rem:projection}, can be done using only $\wtilde{O}(d)$ expected \# of calls to a LOO for $\K$:
	\begin{proposition}
		\label{prop:eucproj}
		Let $\epsilon\in(0,1)$, $b>0$, and $\cU\coloneqq \sqrt{1-\epsilon} \K^\circ$. If $\K$ satisfies \eqref{eq:sandwitch} with $r,R>0$, then there exists a randomized algorithm $\cA^{\cU}_{\mathsf{proj}}$ that given any $(\w, \gamma, \delta) \in \reals^d \times(0,b] \times (0,1)$, outputs $\hat\u \in  \cU$ s.t.~$\|\gamma \hat\u-\w\|^2- \inf_{\u\in\cU} \|\gamma \u-\w\|^2\leq \delta$, with prob.~$\geq 1-\delta$, using $O(d \cdot  \ln \rho)$ expected \# of calls to an LOO for $\K$ and $O(d^3  \ln^{O(1)}(\rho))$ additional arithmetic operations, where $\rho \coloneqq d  R (b/r+ \|\w\|)/(r\delta )$. 
	\end{proposition}
	This result follows from I) a recent result on the Oracle complexity of convex optimization (which subsumes Euclidean projection) using a Separation Oracle \cite{lee2018efficient}; and II) a Separation Oracle for $\K^\circ$ (or $\sqrt{1-\epsilon}\K^\circ$) can be implemented with one call to an LOO for $\K$ \cite{mhammedi2021efficient}. The proof is Appendix \ref{sec:eucproj}. 
	
	In light of Rem.~\ref{rem:projection}, the approximate projection algorithm onto $\sqrt{1-\epsilon}\K^\circ$ from Prop.~\ref{prop:eucproj} brings us a step closer to being able to perform efficient LO on $\K_\epsilon$. We still need a 1d optimization algorithm to minimize the convex function $\Theta$ in Lem.~\ref{lem:oned}. For this, we use the Golden-section search algorithm $\cA_{\mathsf{GS}}$ (Alg.~\ref{alg:gs} in App.~\ref{sec:gs}) and leverage the Lipschitzness of $\Theta$ (see Lem.~\ref{lem:oned}) to show the following:
	\begin{lemma}
		\label{lem:gs}
		Let $\epsilon$, $\K$, $r,R$, $\cU$ and $\cA_{\mathsf{proj}}^{\cU}$ be as in Prop.~\ref{prop:eucproj} and $\kappa\coloneqq R/r$. Let $\w\in\reals^d\setminus \{\bm{0}\}$, $\delta \in(0,1)$, and $\Theta$ be as in Lem.~\ref{lem:oned}. There exists an algorithm $\cA_{\mathsf{GS}}$ (Alg.~\ref{alg:gs}) that given input $(\w, \delta)$ outputs $\hat \gamma$ such that $\P[|\Theta(\hat \gamma)- \inf_{\gamma \geq0}\Theta(\gamma)|\leq r^2\delta]\geq 1 - \delta$ using $K\leq  3 \ln ({(8\kappa^2  + 4 \kappa) \|\w\|^2}/{\delta}) \eqqcolon N$ calls to $\cA_{\mathsf{proj}}^{\cU}$ with inputs $(\w, \gamma_1, \frac{\delta}{5N}), \dots,(\w, \gamma_K, \frac{\delta}{5N})$, for some $\gamma_i \in[0,R\|\w\|], i\in[K]$. 
	\end{lemma} 
	Using Lemma \ref{lem:optCeps} and the algorithm $\cA_{\mathsf{GS}}$ from Lemma \ref{lem:gs}, we now have all the ingredients to build our LOO for $\K_{\epsilon}$.  Our candidate LOO is displayed in Algorithm \ref{alg:LOOCeps}, whose guarantee we now state:
	\begin{lemma}[Weak LOO on $\K_{\epsilon}$]
		\label{lem:linopt}
		Suppose that $\K$ satisfies \eqref{eq:sandwitch} with $r,R>0$, and let $\kappa \coloneqq R/r$. Given any $\w\in \reals^d\setminus \{\bm{0}\} $, $\epsilon>0$, and $\delta\in(0,1)$ such that $\delta \leq \frac{\epsilon^4}{29^4 \kappa^{12}}$, Algorithm \ref{alg:LOOCeps} outputs $\tilde \v$ such that
		\begin{align}
			\mathbb{P}\left[	\tilde \v \in \K_{\epsilon} \quad \text{and}\quad \exists \v_* \in \argmin_{\v\in \K_\epsilon}\inner{\w}{\v}, \text{ s.t. }  \|\tilde \v - \v_*\| \leq {484^4 	 R\delta^{1/4}   \kappa^{32}}{\epsilon^{-4}} \right] \geq 1 - 2\delta, \label{eq:target}
		\end{align}
		using $O(d \cdot \ln (d\kappa/\delta))$ expected number of calls to an LOO for $\K$.
	\end{lemma}
	For $\epsilon, \delta>0$, we define the \emph{Follow-The-Approximate-Leader} algorithm $\cA_{\mathsf{FTAL}}(\epsilon, \delta)$ (approximation of $\cA_{\mathsf{FTL}}(\K_{\epsilon})$) as being Alg.~\ref{alg:boundarywolf} with the step $\w_{t+1}\in \argmin_{\w\in \K}\sum_{s=1}^t \w^\top \g_s$ replaced by $\w_{t+1}= \cO_{\K}(\sum_{s=1}^t \g_s, \epsilon, \delta)$ (Alg.~\ref{alg:LOOCeps}). We now state the guarantee of Alg.~\ref{alg:EffAlg} with subroutine $\cB =\cA_{\mathsf{FTAL}}(\epsilon, \delta)$:
	\begin{algorithm}[t]
		\caption{Weak LOO on $\K_\epsilon$. $\cO_{\K}(\w,\epsilon, \delta)$}
		\label{alg:LOOCeps}
		\begin{algorithmic}[1]
			\REQUIRE Inputs $\w\in \reals^d, \epsilon>0, \delta \in(0,1)$. Instance of $\cA^{\cU}_{\mathsf{proj}}$ [resp.~$\cA_{\mathsf{GS}}$] from Prop.~\ref{prop:eucproj} [resp.~Lem.~\ref{lem:gs}].  
			\vspace{-0.4cm}
			\STATE \textbf{If} $\|\w\|= 0$, \textbf{Then} Return $\tilde \v=\bm{0}$; \textbf{Else} Set $\bar \w = \w/\|\w\|$ \algcomment{\tiny $\argmin_{\v\in \K_{\epsilon}}\inner{\v}{\w}=\argmin_{\v\in \K_{\epsilon}}\inner{\v}{\bar \w}$.} 
			\STATE Set $\hat \gamma =\cA_{\mathsf{GS}}(\bar \w, \delta)$.
			\STATE Set $\hat \u =  \cA_{\mathsf{proj}}^{\cU}(\bar \w,\hat \gamma,\delta)$. \algcomment{We recall that $\cU = \sqrt{1-\epsilon}\K^\circ$.}
			\STATE Set $\hat \z= \bar \w-\hat \gamma \hat \u $ and $\hat \lambda = \hat \gamma^2 + r^2 \|\hat \z \|^2$.
			\STATE Set $\hat \v = \hat \lambda \hat \z/\inner{\hat \z}{\bar \w}$ and  $\theta =1 +   \frac{576^2 \delta^{1/4}\kappa^{15}}{\epsilon^2}$.
			\STATE Return $\tilde \v = \hat \v/\theta$.
		\end{algorithmic}
	\end{algorithm}

	\begin{theorem}
		\label{thm:master2}
		Let $\rho\in(0,1)$. Suppose $\K$ satisfies \eqref{eq:sandwitch} for $r,R>0$, and let $\kappa \coloneqq {R}/{r}$. Let $(\x_t)$ be the outputs of Alg.~\ref{alg:EffAlg} in response to any sequence of subgradients $(\g_t\in \partial f_t(\x_t))$. If Alg.~\ref{alg:EffAlg} is run with $\cB=\cA_{\mathsf{FTAL}}(\epsilon, \delta)$ and params.~$\eta,\mu, L, \epsilon,\delta>0$ s.t.~$\max_{t\in[t]}\|\g_t\|\leq L$, $2R L\eta \leq 1$, $\mu =\frac{2 \epsilon}{r \sqrt{1+\epsilon \kappa^2}}$, and $\delta = \frac{\rho \epsilon^{16}}{484^{16} T \kappa^{128}}$, then for $V_T \coloneqq \sum_{t=1}^T \|\g_t\|^2$, with probability $\geq 1 - \rho$, $(\x_t)\subset \K$ and $\forall \u \in \K$,	
		\begin{align}
			\sum_{t=1}^T \inner{\g_t}{\x_t - \u} \leq  L+ L R T \kappa^2 \epsilon + \eta^{-1}{\ln (T+1)} + \frac{3 \eta R^2}{4}V_T  +\wtilde O\left(R\sqrt{{(1+\kappa^2 \epsilon)^{1/2} V_T }/{\epsilon} }\right).
		\end{align}
	\end{theorem}
	
	\paragraph{Parameter Tuning and Oracle Complexity.}
	By setting $\epsilon = 1/(\kappa^{4/3} T^{1/3})$ and $\eta = 1/(LR\sqrt{T})$ in the setting of Theorem~\ref{thm:master2}, we obtain a regret bound of the form $\wtilde{O}(\kappa T)^{2/3}$. By applying a certain reparametrization to $\K$ if necessary, the constant $\kappa$ may be assumed to be less than $d$ [resp.~$\sqrt{d}$] for general [resp.~centrally-symmetric] sets, without loss of generality (see e.g.~\cite{flaxman2005online,mhammedi2021efficient}).
	For the Oracle complexity, we have that by Lemma \ref{lem:linopt} and our choice of $\delta$ in Theorem~\ref{thm:master2}, the instance of Alg.~\ref{alg:EffAlg} in the setting of Theorem~\ref{thm:master2} makes $O(d \ln  (d\kappa  T /(\rho\epsilon)))$ expected number of calls to an LOO for $\K$. 
	
	\section{Discussion}
	\label{sec:discussion}
	In this paper, we presented the first projection-free algorithm to guarantee the optimal (up to-log-factors) $\wtilde{O}(\sqrt{T})$ regret bound in OCO with strongly convex sets, while using two LOO calls per round. Furthermore, this regret bound is \emph{dimension-free}, unlike those of some existing projection-free algorithms \cite{hazan2020,mhammedi2021efficient,garber2022new}. Our method leads to a $\wtilde{O}(T^{2/3})$ regret bound for general convex sets---improving over the previous best $O(T^{3/4})$ by \cite{hazan2012,garber2022new}---at the cost of $\wtilde O(d)$ more LO Oracle calls in expectation. This provides the practitioner with a new non-trivial way to trade-off computation for performance, which should be particularly useful when $\K$ is specified by an LOO (e.g.~in combinatorial optimization). Since our algorithm for general convex sets uses projections onto $\K^\circ$ (see Rem.~\ref{rem:projection}), one may wonder why not implement projections on $\K$ directly. The answer is that projections onto $\K^\circ$ can be done more efficiently using $\wtilde{O}(d)$ expected \# of calls to an LOO for $\K$ (see Prop.~\ref{prop:eucproj}). In Appendix~\ref{app:projection}, we show that projections onto $\K$ can be done using $\wtilde{O}(d^2)$ expected \# of calls to an LOO for $\K$ and an additional $\wtilde{O}(d^4)$ arithmetic operations. Combining this observation with Thm.~\ref{thm:master2} and the $O(T^{3/4})$ regret due to \cite{hazan2012}, one can conclude that a $\wtilde O(T^{\frac{1+i}{2+i}})$ regret is achievable for general OCO using $\wtilde O(d^{2-i})$ LOO calls/round, for $i\in \{0,1,2\}$. 
	Finally, we note that our algorithms currently require a scale parameter $L$ as input. However, using recent reductions due to \cite{mhammedi2019,mhammedi2020}, our algorithms can easily be made scale-free.

	\clearpage

	\section*{Acknowledgment}
		We thank Adam Block, Sasha Rakhlin, and Max Simchowitz for their helpful comments on the presentation. We acknowledge support from the ONR through awards N00014-20-1-2336 and N00014-20-1-2394.

	\bibliography{biblio}
	\bibliographystyle{alpha}
	\clearpage
	
	\clearpage
	\appendix
	
	\tableofcontents
	\clearpage

	\section{Algorithms Based on Linear Optimization vs Separation/Membership Evaluation}
	\label{sec:oracles}
	As observed by \cite{mhammedi2021efficient}, when it comes to the computational complexity of the Oracles, projection-free algorithms that use LO Oracles (e.g.~FW-style algorithms \cite{hazan2012,hazan2020}) and those that use Separation/Membership Oracles (e.g.~\cite{levy2019,mhammedi2021efficient,garber2022new}) may be viewed as complementary. This is because there are convex sets $\K$ on which LO can be done more efficiently than Membership Evaluation (e.g.~Matroid polytopes \cite{hazan2012}), and vice versa. The complexity of Membership Evaluation on $\K$ is essentially equivalent to the complexity of LO on $\K^\circ$ (see Appendix \ref{app:projection}), the polar set of $\K$ (see e.g.~\cite{mhammedi2021efficient}). 
	This, together with the fact that $(\K^\circ)^\circ=\K$ for a closed convex set $\K$, means that LO on $\K$ can be viewed as testing membership/performing separation on $\K^\circ$ (for a closed convex $\K$). Therefore, when it comes to the computational complexity of the Oracles, algorithms that use an LOO and those that use Separation/Membership Oracles may be viewed as complementary.

	\section{Proofs}
	\label{sec:proofs}
	\subsection{Proof of Lemma \ref{lem:FTLregret}}
	\begin{proof}[{\bf Proof of Lemma \ref{lem:FTLregret}}]
		\label{app:FTLregret}
		Let $\G_t \coloneqq \sum_{s=1}^t \g_s$. By the standard regret decomposition of FTL, we have $\sum_{t=1}^T\inner{\g_t}{\w_t - \u} \leq \sum_{t=1}^T\inner{\g_t}{\w_t - \w_{t+1}}$, for all $\u \in \K$ (see e.g.~\cite{cesa2006prediction}). We now bound the RHS of this inequality. Since $\w_t \in \argmin_{\w\in \K} \w^\top \G_t = \partial \sigma_{\K}(\G_t)$ \cite{hiriart2004}, we have by the Cauchy Schwarz inequality and Lemma \ref{lem:strong},
		\begin{align}
			\sum_{t=1}^T\inner{\g_t}{\w_t - \w_{t+1}} \leq \sum_{t=1}^T\|{\g_t}\|\|{\w_t - \w_{t+1}}\| \leq 2 \sum_{t=1}^T\frac{\|{\G_{t-1}- \G_{t}}\|^2}{\mu (\|\G_{t}\| +\|\G_{t-1}\|)}.
		\end{align} 
		This implies the desired result after dropping $\|\G_{t-1}\|$ from the denominator.
	\end{proof}
	
	\subsection{Proof of Lemma \ref{lem:parameterfree-frankwolfe}}
	\label{app:parameterfree-frankwolfe}
	\begin{proof}[{\bf Proof of Lemma \ref{lem:parameterfree-frankwolfe}}]
		To obtain the desired equality it suffices to add and subtract $\sum_{t=1}^T  \gamma_{\K}(\u) \inner{\g_t}{\w_t}$ from the RHS of \eqref{eq:subtract}, where $\gamma_{\K}$ is the Gauge function of the set $\K$ as in Definition \ref{def:supportfun}.
	\end{proof}
	
	\subsection{Proof of Theorem \ref{thm:master}}
	\label{sec:proofmain}
	\paragraph{Sleeping Experts.}
	To prove the regret bound of Theorem \ref{thm:master}, we will use a result from the Sleeping Experts' setting that we now review. In this setting, there are $k\in[N]$ experts, each of whom is either ``active'' or ``asleep'' at any given round. In particular, at the beginning of any round $t\geq 1$,  the learner receives a vector $\bm{I}_t\in \{0,1\}^N$ (which can be an arbitrary function of the past), such that $I_{t,k}=1$ if and only if expert $k\in [N]$ is active at round $t$. In this setting, we want an algorithm whose outputs $(\p_t)\subset \Delta_N$ guarantee a small regret $\sum_{t=1}^T I_{t,k} (\p_t^\top \bm{\ell}_t- \ell_{t,k})$ against expert $k$ on the rounds where they were active, where $\bm{\ell}_t\in \reals^N$ is the vector of experts' losses at round $t$. This can be achieved thanks to a known technique that reduces the Sleeping Experts' setting to the standard experts' setting \cite{adamskiy2012closer,gaillard2014second}. Next, we describe this reduction.
	
	Given any base algorithm $\cB$ for the standard experts' setting with $N$ experts, the technique consists of feeding the experts' algorithm $\cB$ surrogate losses $\tilde \ell_{t,k} =I_{t,k} \ell_{t,k} +(1-I_{t,k})\p_t^\top \bell_t$, where $\p_t$ is the weight vector played at round $t$ and $\bell_t$ is the observed loss vector. The weight vector $\p_t$ is such that $p_{t,k} = \frac{I_{t,k}  \pi_{t,k}}{\sum_{i=1}^N I_{t,i}  \pi_{t,i}},$
	where $\bm{\pi}_{t}$ is the output vector of the base algorithm $\cB$ at round $t$ given the history of surrogate losses $(\tilde {\bm\ell}_s)_{s<t}$. This strategy is summarized in Algorithm \ref{alg:sleeping}, whose guarantee is as follows:
	\begin{lemma}
		\label{lem:sleeping}
		Let $R_{s}^{\cB}(\e_k)$ denote the regret after $s\geq 1$ rounds of the base algorithm $\cB$ within Alg.~\ref{alg:sleeping} against expert $k\in[N]$. Then, the outputs of Alg.~\ref{alg:sleeping} in response to any sequence of loss vectors $(\bell_t)$ satisfy
		\begin{align}
			\forall s\geq1,	\forall k \in[N],\quad 	\sum_{t=1}^s I_{t,k} (\p_t^\top \bell_t-\ell_{t,k}) \leq \sum_{t=1}^s  \bm{\pi}_{t}^\top\tilde{\bm{\ell}}_t- \tilde \ell_{t,k} = R_{s}^{\cB}(\e_k).
		\end{align}
	\end{lemma}
	This bound follows by a standard Sleeping Experts' reduction (see e.g.~\cite{gaillard2014second}). When Algorithm \ref{alg:sleeping} is instantiated with $\cB$ set to the Hedge algorithm \cite{freund1997decision} $\cA_{\mathsf{Hedge}}(\eta)$  with learning rate $\eta\in(0,1)$; that is, when $\bm\pi_{t+1}$ in Line \ref{line:hedge} of Alg.~\ref{alg:sleeping} is such that $\pi_{t+1,k}\propto e^{-\eta \sum_{s=1}^t \tilde \ell_{t,k}}$ for all $k$, then we have the following guarantee which follows from Lemma~\ref{lem:sleeping} and the regret bound of the Hedge algorithm:
	\begin{algorithm}
		\caption{Sleeping Experts' Reduction.}
		\label{alg:sleeping}
		\begin{algorithmic}[1]
			\REQUIRE Base expert algorithm $\cB$, parameter $\eta >0$, and number of experts $N\in\mathbb{N}$. 
			\STATE  Set $\bm{\pi}_1$ to the first output of $\cB$. 
			\FOR{$t=1,2,\dots$}
			\STATE Get vector $\bm{I}_t\in \{0,1\}^N$ from the environment. \algcomment{Determines the `awake' state of experts.}
			\STATE Play $\p_{t}$ such that $p_{t,k} = \frac{I_{t,k}  \pi_{t,k}}{\sum_{i=1}^N I_{t,i}  \pi_{t,i}}$, for all $k\in[N]$, and observe loss vector $\bm{\ell}_t$.\label{step:5}
			\STATE Define the surrogate losses $\tilde \ell_{t,k} \coloneqq I_{t,k} \ell_{t,k} + (1-I_{t,k})\p_t^\top \bm{\ell}_t$, for every $k\in[N]$.
			\STATE Get output $\bm\pi_{t+1}$ from $\cB$ given the history $(\tilde{\bell}_s)_{s\leq t}$. \label{line:hedge}
			\ENDFOR
		\end{algorithmic}
	\end{algorithm}
	\begin{align}
		\sum_{t=1}^s I_{t,k} (\p_t^\top \bell_t-\ell_{t,k}) \leq \sum_{t=1}^s  \bm{\pi}_{t}^\top\tilde{\bm{\ell}}_t- \tilde \ell_{t,k} \leq \frac{\ln N}{\eta} + \frac{3\eta }{4}\sum_{t=1}^s \text{Var}_{i\sim \bm\pi_{t}}[\tilde\ell_{t,i}]\leq \frac{\ln N}{\eta}  +\sum_{t=1}^s \frac{3\eta \|\tilde\bell_{t}\|^2_{\infty}}{4},  \label{eq:reductionregret}
	\end{align}
	for all $k\in [N]$, $s\geq 1$, and $\eta >0$ such that $\eta \cdot(\max_{i\in[N]}\ell_{t,i} - \min_{j\in[N]} \ell_{t,j})\leq 1$, for all $t\geq 1$. The bound in \eqref{eq:reductionregret} follows from \cite[(5) and Lemmas 1 and 4]{de2014follow}. We will use this guarantee to prove Theorem \ref{thm:master} next:
	\begin{proof}[{\bf Proof of Theorem \ref{thm:master}}] 
		First, we will show that Algorithm \ref{alg:EffAlg} is an instance of the Sleeping Experts' algorithm (Alg.~\ref{alg:sleeping}) with $\cB\equiv \cA_{\mathsf{Hedge}}(\eta)$ and where the experts' losses correspond to the losses of instances of $\cA_{\FTSL}$ (Alg.~\ref{alg:comparatoradaptive}) initialized at different rounds. This will allow us to invoke \eqref{eq:reductionregret} for certain intervals and show the desired regret bound. 
		
		\paragraph{Alg.~\ref{alg:EffAlg} as instance of sleeping experts.} We now describe an instance of Algorithm \ref{alg:sleeping} that uses $\cB \equiv \cA_{\mathsf{Hedge}}(\eta)$ and whose losses match those of Algorithm~\ref{alg:EffAlg}. This will allow us to invoke Lemma \ref{lem:sleeping} to bound the regret of the latter. Let $(\x_t)$ be the outputs of Algorithm~\ref{alg:EffAlg} and $\tilde \g_t \coloneqq \g_t \cdot \mathbf{1}\{\|\g_t\|\geq L/t\}$, where $\g_t \in \partial f_t(\x_t)$, for all $t\geq 1$. We will bound the `pseudo' regret $\sum_{t=1}^T \inner{\tilde \g_t}{\x_t - \u}$ as this will immediately imply a bound of the regret thanks to \eqref{eq:reduction}.

		Consider the instance of Algorithm \ref{alg:sleeping} with $N= T+1$ experts such that only two are awake at any given round $t\in[T]$; namely, expert $N$ and expert $\tau \in [N-1]$, where $$\tau= \max \cT_t; \quad   \cT_t \coloneqq   \{ s\in [t]\colon  R  \mu  \left\|\G_{s-1}\right\|^2 \leq   Q_{s-1}  \ln T\}; $$
		and $(\G_s,Q_s)\coloneqq (\sum_{i=1}^s \tilde \g_i, \sum_{i=1}^s \|\tilde \g_i\|^2$). Let $I_{t,N}=1$, for all $t\in[T]$. Further, for any $k\in[N]$, let $I_{t,k}= \mathbb{I}\{k=\max \cT_t\}$ and $\cI_k\coloneqq \{t\in[T]\colon I_{t,k}=1 \}$. It is easy to verify that $\cI_k$, which represents the subset of rounds where expert $k$ is awake, is an interval. 
		
		Now that we have described the sleep/awake schedule of the experts, we move to the definition of their losses $(\ell_{t,k})$. For $k\in[N-1]$ and $t\in\cI_k$, let $\v_{t,k}$ be the $(t-k+1)^{\text{th}}$ output of $\cA_{\FTSL}$ (Alg.~\ref{alg:comparatoradaptive}) given loss history $(f_s\colon\w\mapsto \inner{\tilde \g_s}{\w})_{k\leq s<t}$. Let $\v_{t,k}=\bm{0}$, for $t<k$. Further, let $\v_{t,N}$ be the $t^{\text{th}}$ output of $\cA_{\FTSL}$ given loss history $(f_s\colon\w\mapsto \inner{\tilde \g_s}{\w})_{s\in[t-1]}$. We define the loss $\ell_{t,k}$ [resp.~surrogate loss $\tilde\ell_{t,k}$] of expert $k\in[N]$ at round $t$ as
		\begin{align}
			\ell_{t,k} \coloneqq \inner{\tilde \g_t}{\v_{t,k}}, \quad \text{and}\quad
			[\text{resp.}\quad \tilde \ell_{t,k} = I_{t,k}\ell_{t,k}+(1- I_{t,k}) \p_t^\top \bell_t],\quad  \forall k\in [N]\label{eq:surrogate} 
		\end{align}
		where $(\p_t)$ is the played weight vector on Line \ref{step:5} of the Sleeping Experts' Algorithm \ref{alg:sleeping}. We now show that the instantaneous losses $(\p^\top_t \bell_t)$ of the instance of Alg.~\ref{alg:sleeping} under consideration match the (linearized) losses $(\inner{\tilde \g_t}{\x_t})$ of Algorithm \ref{alg:EffAlg}, which will allow us to take advantage of Lemma \ref{lem:sleeping} and derive a regret bound for the latter. 
		
		According to the definitions of $(\v_{t,k})$, one can check that the iterates $\u_t$ and $\w_t$ in Algorithm~\ref{alg:EffAlg} satisfy $\u_t = \v_{t,N}$ and $\w_t=\v_{t,\tau}$ for $\tau =\max \cT_t$. Furthermore, since the instance of Alg.~\ref{alg:sleeping} under consideration (whose outputs we denoted by $(\p_t)$) uses $\cB\equiv \cA_{\mathsf{Hedge}}(\eta)$, one can verify that at any round $t$, the scalar weight $q_t$ in Algorithm \ref{alg:EffAlg} satisfies
		\begin{align}
			(q_t, 1- q_t) = (p_{t,N}, p_{t,\tau})^\top \in \Delta_2,
		\end{align}
		Therefore, we have $\inner{\x_t}{\tilde \g_t}=\inner{q_t\u_t+(1-q_t)\w_t}{\tilde \g_t}= q_t\ell_{t,N}+ (1-q_t)\ell_{t,\tau}=\p_t^\top \bell_t$, where the last equality follows by the fact that only experts $N$ and $\tau = \max \cT_t$ are awake at around $t$. This shows that, indeed, the loss $\inner{\tilde \g_t}{\x_t}$ of Alg.~\ref{alg:EffAlg} matches that of the instance of Alg.~\ref{alg:sleeping} under consideration. Using this fact, we can now invoke Lemma \ref{lem:sleeping} (the Sleeping Experts' regret bound) and \eqref{eq:reductionregret} to bound the regret of Alg.~\ref{alg:EffAlg}. We will invoke the bound in \eqref{eq:reductionregret} on two intervals; $[1,\bar \tau-1]$ and $[\bar \tau, T]$, where 
		$$\bar \tau \coloneqq \max \cT_T,$$
		which represents the round $\tau$ such that, for all $t\geq \tau$, $\|\sum_{s=1}^{t}\tilde \g_s\|^2 > (R \mu)^{-1}\sum_{s=1}^{t}\|\tilde \g_s\|^2 \cdot \ln T$.
		\paragraph{Regret on the interval $[1..\bar \tau -1].$} By definition of $(\tilde\bell_t)$ in \eqref{eq:surrogate}, we have
		\begin{align}
			|\tilde\ell_{t,k}| \leq R \|\tilde \g_t\|.  \quad \text{for all }k\in[N]. \label{eq:boundedloss}
		\end{align}
		Thus, the choice of $\eta$ in the theorem statement ensures that $\eta \cdot (\max_{i\in[N]} \ell_{t,i} - \min_{j\in[N]} \ell_{t,j})\leq 1$, which means we can in 	instantiate the regret bound in \eqref{eq:reductionregret} against expert $k \in[N]$. Doing so for expert $k=N$ (who is awake in all rounds) between the rounds $s=1$ and $\bar \tau-1$, we get (recall that $N=T+1$)
		\begin{align}
			\sum_{t=1}^{\bar \tau -1} \inner{\tilde \g_t}{\x_t-\u_t}=\sum_{t=1}^{\bar \tau -1} (\p_t^\top \bell_t -\ell_{t,N})\leq \frac{\ln (T+1)}{\eta} + \frac{3\eta}{4} \sum_{t=1}^{\bar \tau-1} \|\tilde \bell_{t}\|_{\infty}^2, \label{eq:part1regret}
		\end{align}
		where the first equality follows by the definition of $(\bell_t)$ in \eqref{eq:surrogate} and the fact that $\v_{t,N}=\u_t$ (see previous paragraph). 
		Using \eqref{eq:boundedloss} again and rearranging \eqref{eq:part1regret}, we get that for all $\u \in \K$, 
		\begin{align}
			\sum_{t=1}^{\bar \tau-1}\inner{\tilde \g_t}{\x_t-\u}&  \leq 	\sum_{t=1}^{\bar \tau-1}\inner{\tilde \g_t}{\u_t-\u} + \frac{\ln (T+1)}{\eta} +\frac{3\eta R^2}{4}\sum_{t=1}^{\bar \tau-1}\|\tilde \g_t\|^2,\nn \\
			& \leq R^{\cA_1}_{1:\bar \tau-1}(\bm{0}) - \u^\top \sum_{t=1}^{\bar \tau -1}\tilde \g_t + \frac{\ln (T+1)}{\eta} +\frac{3\eta R^2}{4}\sum_{t=1}^{\bar \tau-1}\|\tilde \g_t\|^2,\nn \\
			& \leq L +\|\u\| \sqrt{R^{-1}\mu^{-1} Q_T \ln T}+ \frac{\ln (T+1)}{\eta} +\frac{3\eta R^2 Q_{T}}{4}, \label{eq:strong}
		\end{align}
		where the last inequality follows by I) the regret bound of $\cA_1$ from Lem.~\ref{lem:parameterfree-frankwolfe} and that of $\cA_{\od}$ at the origin; II) $\gamma_{\K}(\bm{0})=0$; and III) the definition of $\bar \tau$ which implies that $\|\sum_{t=1}^{\bar \tau-1} \tilde \g_t\|\leq \sqrt{R^{-1}\mu^{-1} Q_{1:\bar \tau-1} \ln T}$, where $Q_{1:\bar \tau-1} \coloneqq \sum_{t=1}^{\bar \tau -1}\|\tilde \g_t\|^2\leq  \sum_{t=1}^{T}\|\tilde \g_t\|^2 \eqqcolon  Q_T$.
		
		\paragraph{Regret on the interval $[\bar \tau..T].$} 
		We now bound the regret between the rounds $\bar \tau$ and $T$. By definition of $\bar \tau$, the instance $\cA_{\bar\tau}$ in Alg.~\ref{alg:EffAlg} is active between the rounds $\bar \tau$ and $T$. Thus, invoking the regret bound in \eqref{eq:surrogate} for the instance $\cA_{\bar \tau}$ between rounds $s=\bar \tau$ and $T$, we get, for any $\u \in \K$,
		\begin{align}
			\sum_{t=\bar \tau}^{T}\inner{\tilde \g_t}{\x_t-\u}&  \leq 	\sum_{t=\bar \tau}^{T}\inner{\tilde \g_t}{\w_t-\u} + \frac{\ln (T+1)}{\eta} +\frac{3\eta}{4}\sum_{t=\bar \tau}^{T} \text{Var}_{k\sim \bm{\pi}_t}[\tilde \ell_{t,k}],\nn \\
			& \leq \sum_{t=\bar \tau}^{T}\inner{\tilde \g_t}{\w_t-\u} + \frac{\ln (T+1)}{\eta} +\frac{3\eta R^2}{4}\sum_{t=\bar \tau}^{T} \|\tilde \g_t\|^2, \quad (\text{by}\ \eqref{eq:boundedloss}) \nn \\
			& \leq  R^{\cA_{\bar\tau}}_{\bar \tau:T}(\bm{u})  + \frac{\ln (T+1)}{\eta} +\frac{3\eta R^2 Q_T}{4},\label{eq:state} \\
		\end{align}
		It remains to bound the regret $R^{\cA_{\bar \tau}}_{\tau:T}(\u)$ of Algorithm \ref{alg:comparatoradaptive} using Lemma~\ref{lem:parameterfree-frankwolfe}. By the regret bounds of $\cA_{\FTSL}$ and $\cA_{\mathsf{FG}}$, we have 
		\begin{align}
			R^{\cA_{\bar\tau}}_{\bar \tau:T}(\u) =\sum_{t=\bar \tau}^T \inner{\tilde \g_t}{\w_t- \u}   & \leq R^{\cA_{\od}}_T(\gamma_{\K}(\u)) + \frac{\gamma_{\K}(\u)}{2\mu} \sum_{t=\bar \tau}^T \frac{\|\tilde \g_t\|^2}{\|\sum_{s=\bar \tau}^t \tilde \g_s\|},\nn \\
			& \leq O(R \sqrt{Q_T \ln T}) + \frac{\gamma_{\K}(\u)}{2\mu} \sum_{t=\bar \tau}^T \frac{\|\tilde \g_t\|^2}{\|\sum_{s=\bar \tau}^t \tilde \g_s\|}, \label{eq:pluto1}
		\end{align}
		We now bound the last term in \eqref{eq:pluto1}. By definition of $\bar \tau$ and the triangle inequality, we have 
		\begin{align}
			\forall t\geq \bar \tau, \quad \left\|\sum_{s=\bar \tau}^t \tilde \g_s \right\| &\geq 	\left\|\sum_{s=1}^t \tilde \g_s \right\|  - 	\left\|\sum_{s=1}^{\bar \tau -1} \tilde \g_s \right\| \geq \sqrt{R^{-1}\mu^{-1} Q_{1:t} \ln T} - \sqrt{R^{-1}\mu^{-1} Q_{1:\bar \tau-1}\ln T}. 
		\end{align}
		Using this, we can bound the last sum in \eqref{eq:pluto1} by
		\begin{align}
			\sum_{t=\bar \tau}^T \frac{\|\tilde \g_t\|^2}{\|\sum_{s=\bar \tau}^t\tilde \g_s\|}  & \leq \sqrt{\frac{R\mu}{\ln T}}\sum_{t=\bar \tau}^T \frac{\|\tilde \g_t\|^2}{\sqrt{Q_{1:t}}- \sqrt{Q_{1:\bar \tau-1}}} \nn \\
			& = \sqrt{\frac{R \mu}{\ln T}}\sum_{t=\bar \tau}^T \frac{\|\tilde \g_t\|^2( \sqrt{Q_{1:t}}+ \sqrt{Q_{1:\bar \tau -1}})}{Q_{\bar \tau:t}},\nn \\
			& \leq 2\sqrt{\frac{R\mu Q_{1:\bar \tau -1}}{\ln T}} \sum_{t=\bar \tau}^T \frac{\|\tilde \g_t\|^2}{Q_{\bar \tau:t}} +  \sqrt{\frac{R\mu }{\ln T}} \sum_{t=\bar \tau}^T \frac{\|\tilde \g_t\|^2}{\sqrt{Q_{\bar \tau:t}}},\nn \\
			& \leq 4\sqrt{\frac{R\mu Q_{T} }{\ln T}}\ln \frac{Q_{\bar \tau:T}}{\|\tilde \g_{\bar \tau}\|^2}+  2 \sqrt{\frac{R\mu Q_{\bar \tau:T}}{\ln T}},\label{eq:this}
		\end{align}
		where the last inequality follows by the fact that for any for positive numbers $\alpha_0,\dots,\alpha_n$,
		\begin{equation}
			\begin{aligned}
				\sum_{i=1}^n \frac{\alpha_i}{\sqrt{\sum_{j=1}^i \alpha_{j}}}&\le 2\sqrt{\sum_{i=1}^n \alpha_i}
				\\
				\sum_{i=1}^n \frac{\alpha_i}{\alpha_0+\sum_{j=1}^i \alpha_{j}}&\le \ln\frac{\alpha_0+\sum_{i=1}^n \alpha_i}{\alpha_0}.
			\end{aligned}
			\label{eq:just}
		\end{equation}
		(See e.g.~\cite{cutkosky2019,levy2018online}.) Finally, note that by definition of $\bar \tau$, we must have $\|\tilde\g_{\bar \tau}\|>0$. This in turn implies that $\|\tilde \g_t\|\geq L/t$ since $\tilde \g_t = \g_t \cdot \mathbf{1}\{\|\g_t\| \geq L/t\}$, and so $Q_{\bar \tau:T}/\|\tilde \g_{\bar \tau}\|^2 \leq T$.
		Combining this with \eqref{eq:this}, \eqref{eq:pluto1}, \eqref{eq:state}, \eqref{eq:strong}, and \eqref{eq:reduction}, (where the latter relates the pseudo regret to the actual regret) we obtain the desired result.
	\end{proof}

	\subsection{Proof of Lemma \ref{lem:approxlemma}}
	\label{sec:approxlemma}
	For the proof of Lemma \ref{lem:approxlemma}, we need the following result:
	\begin{lemma}
		\label{lem:midpointstrong}
		Let $\K_{\epsilon}$ be as in \eqref{eq:approx} and define $\psi\colon \w \mapsto\gamma_{\K_{\epsilon}}(\w)^2$. Then, $\psi$ is $\epsilon/r^2$ midpoint strongly convex; that is, for any $\x, \y \in \reals^d$,
		\begin{align}
			\psi\left(\frac{\x+\y}{2}\right)\leq \frac{	\psi(\x) + \psi(\y)}{2} - \frac{\epsilon}{4 r^2}\|\x- \y\|^2.
		\end{align}
	\end{lemma}
	\begin{proof}[{\bf Proof}]
		This proof closely follows steps in \cite{molinaro2020}. By definition of the operator $\oplus$ and the fact that for any convex set $\cK$, $\sigma_{\cK^{\circ}}(\cdot)= \gamma_{\cK}(\cdot)$ (see e.g.~\cite{molinaro2020}), we have $\gamma_{\K_{\epsilon}}^2(\u) = (1-\epsilon)\gamma_{\K}^2(\u) +\epsilon \|\u\|^2 /r^2$, for any $\u \in \reals^d$. Thus, we have 
		\begin{align*}
			\psi\left(\frac{\x+\y}{2}\right) & = 		\gamma_{\K_\epsilon}\left(\frac{\x+\y}{2}\right)^2, \\ & = (1-\epsilon) \gamma_{\K}\left(\frac{\x+\y}{2}\right)^2 + \frac{\epsilon}{r^2} \left\|\frac{\x+\y}{2}\right\|^2, \\
			& \le  (1-\epsilon) \left(\frac{\gamma_{\K}(\x)^2}{2} + \frac{\gamma_{\K}(\y)^2}{2} \right) + \frac{\epsilon}{r^2}  \left\|\frac{\x+\y}{2}\right\|^2, \ \ \text{(Jensen's inequality)}\\
			& = (1-\epsilon) \left(\frac{\gamma_{\K}(\x)^2}{2} + \frac{\gamma_{\K}(\y)^2}{2} \right) + \frac{\epsilon}{r^2} \left(\frac{\|\x\|^2}{2} + \frac{\|\y\|^2}{2} - \frac{\|\x - \y\|^2}{4}\right),\\
			& = \frac{\gamma_{\K_\epsilon}(\x)^2}{2} + \frac{\gamma_{\K_\epsilon}(\y)^2}{2} - \frac{\epsilon}{4r^2} \|\x -\y\|^2_2,
		\end{align*}
		where the penultimate equality follows by the fact that $\|\x+\y\|^2 =2 \|\x\|^2 + 2 \|\y\|^2 - \|\x-\y\|^2$.
	\end{proof}
	\begin{proof}[{\bf Proof of Lemma \ref{lem:approxlemma}}]
		The second point follows from \cite[Theorem 6]{molinaro2020}. For the first point, we use the fact that the function $\psi\colon \w \mapsto \gamma_{\K_{\epsilon}}(\w)^2$ is $\epsilon/r^2$ is midpoint strongly convex (see Lemma \ref{lem:midpointstrong}). Now, by \cite[Theorem 2.3]{azocar2011strongly}, the function $\psi$ is $2\epsilon/r^2$ strongly convex; that is, for any $\x, \y \in \reals^d$ and $\lambda \in [0,1]$,
		\begin{align}
			\psi( \lambda \x+(1-\lambda) \y)  \leq \lambda 	\psi(\x) + (1-\lambda)	\psi(\y) - \frac{2 \epsilon}{r^2}\lambda (1-\lambda) \|\x - \y\|^2.\label{eq:new}
		\end{align}
		Thus, for any $\x, \y \in \K_{\epsilon}, \lambda \in [0,1]$ and $\v \in \reals^d$, we have, by subadditivity of the Gauge function
		\begin{align}
			\gamma_{\K_\epsilon}(\lambda \x + (1-\lambda) \y +\v) &\leq \gamma_{\K_\epsilon}(\lambda \x + (1-\lambda) \y) + \gamma_{\K_{\epsilon}}(\v),\nn \\
			& \stackrel{\eqref{eq:new}}{\leq} \sqrt{ \lambda 	\gamma_{\K_\epsilon}(\x)^2 + (1-\lambda)	\gamma_{\K_\epsilon}(\y)^2 - \frac{2 \lambda (1-\lambda) \epsilon}{r^2} \|\x - \y\|^2}+  \gamma_{\K_{\epsilon}}(\v),\nn\\
			& \leq \sqrt{1 - \frac{2 \lambda (1-\lambda)\epsilon}{r^2} \|\x - \y\|^2} + \sqrt{1+\epsilon \kappa^2}\frac{\|\v\|}{r}, \label{eq:teeth}\\
			& \leq 1 - \frac{ \lambda (1-\lambda)\epsilon}{r^2} \|\x - \y\|^2 + \sqrt{1+\epsilon \kappa^2}\frac{\|\v\|}{r},\label{eq:need}
		\end{align}
		where \eqref{eq:teeth} follows by the facts that I) $\gamma_{\K_\epsilon}(\x)\vee \gamma_{\K_{\epsilon}}(\y) \leq 1$ since $\x,\y \in \K_{\epsilon}$ II) $\mathbb{B}(r/\sqrt{1+\epsilon \kappa^2}) \subseteq \K_{\epsilon}$ (implies by the second point of the lemma), and III) $\gamma_{\cK}(\u)\leq \|\u\|/r'$ whenever $\mathbb{B}(r')\subseteq \cK$ (see \cite[Lemma 2]{mhammedi2021efficient}); and \eqref{eq:need} follows by the fact that $\sqrt{1-x}\leq 1-x/2$, for $x>0$. By \eqref{eq:need}, if $\v$ is such that $\|\v\|  \leq \epsilon \lambda (1-\lambda)/(r\sqrt{1+\epsilon \kappa^2})$, we have 
		\begin{align}
			\gamma_{\K_\epsilon}(\lambda \x + (1-\lambda) \y +\v) \leq 1,
		\end{align}
		and so $\lambda \x +(1-\lambda)\y + \v\in \K_{\epsilon}$. Therefore, $\K_{\epsilon}$ is $2 \epsilon/(r \sqrt{1+\epsilon \kappa^2})$-strongly convex w.r.t.~$\|\cdot\|$.
	\end{proof}
	
	\subsection{Proof of Lemma \ref{lem:optCeps}}
	\label{sec:optCeps}
	\begin{proof}[{\bf Proof of Lemma \ref{lem:optCeps}}]
		To simplify notation, let $\cU\coloneqq \sqrt{1-\epsilon}\K^\circ$ and $\cK \coloneqq \K^\circ_\epsilon=\cU \oplus \sqrt{\epsilon} \mathbb{B}(1/r)$. For $\lambda_*$ as in the lemma's statement, our first step in this proof is to show that $\lambda_* = \gamma_{\K^\circ_\epsilon}(\w)$, which reduces to showing that $\lambda_* \in \inf\{\lambda >0\colon \w\in \lambda \cK\}$. Thanks to a known characterization of the $L_2$ set addition (see e.g.~\cite{lutwak2012brunn,molinaro2020}), we can write $\cK$ as
		\begin{align}
			\cK= \{ \sqrt{1-\alpha} \u  + \sqrt{\alpha} \v \colon \alpha \in[0,1], \u \in \cU, \v \in \mathbb{B}(1/r)\}. \label{eq:char}	
		\end{align}
		Thus, if $\w\in \lambda \cK$, for $\lambda >0$, then there exists $\alpha \in[0,1], \u \in \cU, \v \in \mathbb{B}(1/r)$ such that $\lambda\sqrt{1-\alpha} \u  + \lambda \sqrt{\alpha} \v =\w$. By letting $\gamma = \lambda \sqrt{1-\alpha}$, the latter equality implies that 
		\begin{align}
			\inf_{\u' \in \cU} \|\gamma \u'  -\w\|^2  \leq 	\|\gamma \u  -\w\|^2  = \lambda^2 \alpha \|\v\|^2 \leq \lambda^2 \alpha/r^2 =  (\lambda^2 - \gamma^2)/r^2. \label{eq:ineq}
		\end{align}
		Thus, rearranging \eqref{eq:ineq} implies that $\lambda \geq \inf_{\gamma\geq 0,\u \in \cU} \Psi(\gamma, \u)^{1/2}=\lambda_*$. We now show that if $\lambda =\lambda_*$ then $\w\in\lambda \cK$. Let $(\gamma_*,\u_*)$ be as in the lemma's statement. First note that if $\lambda_*=0$ we must necessarily have $\gamma_*=0$ which in turn implies that $\w= \bm{0}$. This is not possible under the assumptions of the lemma, and so $\lambda_*>0$. Define $\alpha_* \coloneqq 1- \gamma^2_*/\lambda^2_*$.  In this case, we have $\w- \lambda \sqrt{1-\alpha_*}\u_* = \w - \gamma_* \u_*$ (we are considering the case where $\lambda = \lambda_*$). Therefore, by definition of $(\gamma_*, \lambda_*,\u_*)$, we have $$\| \w -  \lambda \sqrt{1-\alpha_*}\u_*\|^2 = \lambda^2/r^2 -\gamma_*^2 /r^2 =  \lambda_*^2 \alpha_* /r^2.$$ This implies that there exists $\v_*\in \mathbb{B}(1/r)$ such that $\lambda \sqrt{\alpha_*} \v_*=\w -  \lambda \sqrt{1-\alpha_*}\u_* $, which is equivalent to 
		\begin{align} 
			\label{eq:rows}
			\lambda \sqrt{1-\alpha_*}\u_*+ \lambda \sqrt{\alpha_*}\v_*= \w.
		\end{align}
		Thus, by \eqref{eq:char}, we have $\w\in \lambda \cK$, and so we conclude that $\lambda_* = \inf\{ \lambda>0\colon \w\in \lambda \cK\}$.
		
		Now, it remains to show the lower bound on $\inner{\z_*}{\w}$. Let $\alpha_*$ and $\v_*$ be as in the previous paragraph. We first establish a lower bound on $\alpha_*$. Since $\lambda_* = \inf\{ \lambda>0\colon \w\in \lambda \cK\}$, we have that $\w$ is on the boundary of $\lambda_* \cK$ and so \eqref{eq:rows} implies that $ \sqrt{1-\alpha_*}\u_*+  \sqrt{\alpha_*}\v_* \in \partial \cK$, where $\partial \cK$ denotes the boundary of $\cK$. On the other hand, we know that $\K^\circ \subseteq\K_\epsilon^\circ =\cK$, which implies that $(1-\epsilon)^{-1/2} \u_* \in \cK$. This rules out the case where $\alpha_* =0$ since this would mean that $\u_* \in  \partial \cK$. Now, let $\x_* \in \argmax_{\x\in \K_\epsilon} \inner{\x}{\w}$. Using \eqref{eq:rows}, \eqref{eq:char}, and that $\lambda_* = \sigma_{\K_\epsilon}(\w)$, we have 
		\begin{align}
			1 =  \inner{\x_*}{\w}/\lambda_* = \inner{\x_*}{ \sqrt{1-\alpha_*}\u_*+  \sqrt{\alpha_*}\v_*}\geq   \inner{\x_*}{ \sqrt{1-\alpha_*}\u_*+  \sqrt{\alpha_*}\v} \label{eq:these}
		\end{align}
		for all $\v \in \mathbb{B}(1/r)$, where the last inequality follows by the fact that $\sqrt{1-\alpha_*}\u_*+  \sqrt{\alpha_*}\v\in \cK$ and $\x_* \in \cK^\circ$. Now, since \eqref{eq:these} holds for all $\v \in \mathbb{B}(1/r)$ and $\lambda_* \cdot \alpha_*>0$, we have $ \inner{\x_*}{\v_*} \geq \inner{\x_*}{\v}, \forall \v \in \mathbb{B}(1/r),$ from which we get that  
		\begin{align}
			\v_* = \frac{\x_*}{r \|\x_*\|}.\label{eq:preinner}
		\end{align}
		Similarly, we have $1=\inner{\x_*}{\w}/\lambda_* = \inner{\x_*}{ \sqrt{1-\alpha_*}\u_*+  \sqrt{\alpha_*}\v_*}\geq   \inner{\x_*}{ \sqrt{1-\alpha_*}\u+  \sqrt{\alpha_*}\v_*}$, for all $\u \in \cU$, which implies that $\inner{\x_*}{\u_*}\geq \inner{\x_*}{\u}$, for all $\u \in \cU$. Setting $\u = \sqrt{1-\epsilon}\x_*/(R\|\x_*\|)$ (which is guaranteed to be in $\cU = \sqrt{1-\epsilon}\K^\circ$ since $\mathbb{B}(1/R)\subseteq \K^\circ$), we get that 
		\begin{align}
			\inner{\x_*}{\u_*} \geq  \inner{\x_*}{\u} = \frac{\sqrt{1-\epsilon}\|\x_*\|}{R}\geq \frac{\sqrt{1-\epsilon}}{\kappa \sqrt{1+\kappa^2\epsilon}}. \label{eq:inner}
		\end{align}
		where the last inequality follows by the fact that $\mathbb{B}(r)/\sqrt{1+\kappa^2\epsilon}\subseteq \K_\epsilon$ (Lemma \ref{lem:approxlemma}) and $\x_* \in \K_{\epsilon}$. Now, using the fact that $(1-\epsilon)^{-1/2}\u_* \in \cK$ as established earlier, we get that  
		$$1=\inner{\x_*}{ \sqrt{1-\alpha_*}\u_*+  \sqrt{\alpha_*}\v_*}\geq \inner{\x_*}{(1-\epsilon)^{-1/2}\u_*}.$$
		By plugging \eqref{eq:preinner} and \eqref{eq:inner} into the inequality in the previous display, we get 
		\begin{align}
			0\leq  	\frac{\sqrt{(1-\epsilon)(1-\alpha_*)}-1}{\kappa \sqrt{1+\kappa^2\epsilon}}+ \frac{\|\x_*\|}{r}\sqrt{\alpha_*} \leq  \frac{\sqrt{(1-\epsilon)(1-\alpha_*)}-1}{\kappa \sqrt{1+\kappa^2\epsilon}}+ \kappa\sqrt{\alpha_*},
		\end{align}
		where the last inequality follows by the fact that $\x_* \in \K_{\epsilon}\subseteq \K \subseteq \mathbb{B}(R)$ (Lemma \ref{lem:approxlemma} and \eqref{eq:sandwitch}).
		This implies that for $\tilde \kappa\coloneqq \kappa^2 \sqrt{1+\kappa^2 \epsilon}$
		\begin{align}
			\alpha_*\geq	\frac{\epsilon ^2-(\tilde\kappa ^2+1) \epsilon +2 \tilde\kappa^2  
				-\tilde\kappa\sqrt{(\epsilon -1) \left(\epsilon -\tilde\kappa ^2\right)}}{\left(\tilde\kappa
				^2-\epsilon +1\right)^2} \geq \frac{\epsilon^2}{4\tilde\kappa^2}. \label{eq:lower}
		\end{align}
		Now, since $\gamma_{*}= \lambda_* \sqrt{1-\alpha_*}$, we have $\z_* = \lambda_* \sqrt{\alpha_*} \v_*$, by \eqref{eq:rows}. Therefore, we have 
		\begin{align}
			\inner{\z_*}{\v_*} =  \lambda_* \sqrt{\alpha_*}  \inner{\v_*}{\w} \stackrel{(*)}{=} \lambda_* \sqrt{\alpha_*}  \frac{\lambda_*}{r\|\x_*\|} \stackrel{(**)}{\geq}  \frac{\lambda^2_*}{r R} \sqrt{\alpha_*} \geq \frac{\lambda^2_*\epsilon}{2 r R \tilde \kappa},\label{eq:plug}
		\end{align}
		where $(*)$ follows by \eqref{eq:preinner}; $(**)$ follows by the fact that $\x_*\in \mathbb{B}(R)$ as argued earlier; and the last inequality we used \eqref{eq:lower}. It remains to bound $\lambda_*$ from below. By Lemma \ref{lem:approxlemma}, we have that $\mathbb{B}(r)\subseteq \K \subseteq \sqrt{1+\kappa^2\epsilon} \K_\epsilon$. Therefore, since $\lambda_* = \gamma_{\cK}(\w) = \sigma_{\K_\epsilon}(\w)$, we have that $\lambda_* \geq \|\w\| r/\sqrt{1+\kappa^2\epsilon}$, which follows from the fact that $\mathbb{B}(r)/\sqrt{1+\kappa^2\epsilon}\subseteq \K_\epsilon$ and \cite[Lemma~2]{mhammedi2021efficient}. Plugging this into \eqref{eq:plug} implies the desired lower bound on $\inner{\z_*}{\v_*}$.
	\end{proof}
	
	\subsection{Proof of Lemma \ref{lem:oned}}
	\label{sec:oned}
	The indicator function $\iota_{\cK}$ of a set $\cK \subseteq \reals^d$ is the function $\iota_{\cK}\colon \reals^d \rightarrow \reals\cup\{+\infty\}$ defined by $\iota_{\cK}(\x) = 0$ if $\x\in \cK$ and $+\infty$ otherwise.
	\begin{proof}[{\bf Proof of Lemma \ref{lem:oned}}]
		We first show convexity. The function $\varphi \colon \w \mapsto \inf_{\u \in \sqrt{1-\epsilon} \K^\circ} \|\u - \w\|$ is convex \cite[Proposition 1]{cutkosky2018black}. Therefore, its perspective transform $(\gamma,\w)\mapsto  \gamma \varphi(\w/\gamma)$ is convex on $\reals_{>0} \times \reals^d$ (see e.g.~\cite[Proposition 2.2.1]{hiriart2004}). In particular, for any $\w\in \reals^d$, $\psi\colon\gamma \mapsto \gamma \varphi(\w/\gamma)= \gamma\inf_{\u \in \sqrt{1-\epsilon} \K^\circ}\|\u - \w/\gamma\| = \inf_{\u \in \sqrt{1-\epsilon} \K^\circ} \|\gamma\u - \w\|$ is convex. Thus, since $x\mapsto x^2$ is convex and non-decreasing on $\reals_{\geq 0}$, the function $\gamma \mapsto \psi^2(\gamma)= \inf_{\u \in \sqrt{1-\epsilon} \K^\circ} \|\gamma\u - \w\|^2$ is convex. Therefore, since $\gamma \mapsto \gamma^2$ is 1-strongly convex, the function $\gamma \mapsto \gamma^2 + r^2\psi^2(\gamma)$ is also $1$-strongly convex. 
		
		Now we show that $\Theta\colon \gamma \mapsto \gamma^2 + r^2 \psi^2(\gamma)$ is differentiable. For this it suffice to show that $\varphi^2$ is differentiable, since $\psi^2(\gamma) =\gamma^2 \varphi^2(\w/\gamma)$. The function $\varphi^2$ can be written as $$\varphi^2(\w) = \inf_{\u \in \reals^d} \left\{\iota_{ \sqrt{1-\epsilon}\K^\circ}(\u)   + \|\u-\w\|^2 \right\},$$ which is the Moreau envelop of the indicator function $\iota_{ \sqrt{1-\epsilon}\K^\circ}$ of the set $ \sqrt{1-\epsilon}\K^\circ$, and so it is $C^1$ \cite{moreau1962fonctions,moreau1965proximite}. Furthermore, the derivative of $\Theta$ is given by
		\begin{align}
			\Theta'(\gamma) = 2 \gamma + 2 r^2 \u^\top_* (\gamma \u_* - \w), \quad \text{for all } \gamma \in(0,b],   
		\end{align}
		where $\u_*\in \argmin_{\u \in \sqrt{1-\epsilon}\K^\circ} \|\gamma \u -\w \|$. Using that $\K^\circ \subseteq \mathbb{B}(1/r)$, we get that 
		\begin{align}
			|\Theta'(\gamma)| \leq 2 (\gamma  +r^2 \gamma \|\u_*\|^2+ r^2 \|\u_*\|\|\w\|) \leq 2 (2b + r \|\w\|), \quad \forall \gamma \in(0,b],
		\end{align} 
		where we used that $\|\u_*\|^2 \leq 1/r^2$ (since $\u_* \in \K^\circ\subseteq \mathbb{B}(1/r)$). This shows the Lipschitzness claim for $\Theta$.
	\end{proof}
	
	\subsection{Proof of Proposition \ref{prop:eucproj}}
	\label{sec:eucproj}
	For the proof of Proposition \ref{prop:eucproj}, we need to introduce a few notions. For a set $\cK \subseteq\reals^d$, we define $\mathbb{B}(\cK,\delta)\coloneqq \{\x \in \reals^d : \exists \y \in \cK  \text{ such that } \|\x-\y\|\leq \delta \}$. We also define $\mathbb{B}(\cK,-\delta)\coloneqq \{\x \in \reals^d : \forall \y \in \mathbb{B}(\delta), \ \x+\y\in  \cK \}$. For a convex function $f\colon \reals^d\rightarrow \reals \cup \{+\infty\}$, we let $\mathrm{dom}(f) \coloneqq \{\x\in \reals^d \colon f(\x)<+\infty\}$ be its domain. We now define the \emph{epigraph} of a convex function:
	\begin{definition}[Epigraph]
		\label{def:epi}
		The epigraph of a function $f\colon \reals^d \rightarrow \reals \cup \{+\infty\}$ is the defined as the set 
		\begin{align}
		\epi(f) \coloneqq \{ (\x,v)\in \reals^d\times \reals\colon  f(\x)\leq v \}.
		\end{align}
		\end{definition}
	We now defined two types of Oracles that will be useful in the proof of Proposition \ref{prop:eucproj}.

	\begin{definition}[Separation Oracle ($\mathsf{SEP}$) \cite{lee2018efficient}]
	\label{def:sep}
	Queried with a vector $\y \in \reals^d$ and a real number $\delta \in(0,1)$, with probability $1-\delta$, the Oracle either
	\begin{itemize}
		\item asserts that $\y \in \mathbb{B}(\cK,\delta)$, or
		\item finds a vector $\c \in \reals^d$ such that $\y \in \mathbb{B}(\cK,\delta)$ and $\c^\top \x \leq \c^\top \y + \delta,$ for all $\x\in \mathbb{B}(\cK,-\delta)$.
	\end{itemize}
	We let $\mathrm{SEP}_{\delta}(\cK)$ be the time complexity of this oracle.
\end{definition}
In the next definition, we let $f\colon \reals^d \rightarrow \reals\cup \{+\infty\}$ be a closed convex function.
\begin{definition}[Subgradient Oracle ($\mathsf{GRAD}$) of $f$ on $\cK$]
	\label{def:grad}
	Queried with a vector $\y$ in some convex subset $\cK\subseteq \reals^d$ and a real number $\delta>0$, the Oracle outputs an extended real number $\alpha$ and a vector $\c\in \reals^d$ such that 
	\begin{gather}
		\forall \x\in \reals^d,\ \	\alpha + \c^\top(\x - \y)< \max_{\z \in \mathbb{B}(\x,\delta)} f(\z) +\delta, \ \
		\text{and}  \ \  \min_{\x\in \mathbb{B}(\y, \delta)}f(\x) -\delta \leq \alpha \leq \max_{\x\in \mathbb{B}(\y, \delta)}f(\x) +\delta.
	\end{gather} 
	We let $\mathrm{GRAD}_{\delta}(f)$ be the time complexity of this oracle.
\end{definition}

In what follows, we let $\mathsf{GRAD}(f)\coloneqq \mathsf{GRAD}_0(f)$ and $\mathsf{SEP}(f)\coloneqq \mathsf{SEP}_0(f)$. The next lemma essentially shows that we can easily build a Separation Oracle for the epigraph of a convex function $f$ using a Subgradient Oracle for $f$ and a Separation Oracle of its domain.
\begin{lemma}
	\label{lem:switch}
Let $b>0$ and $f\colon \reals^d \rightarrow \reals \cup \{+\infty\}$ be a closed convex function s.t.~$\sup_{\x \in \mathrm{dom}(f)}f(\x)\leq b$. Further, let $\cK \coloneqq \mathrm{dom}(f)$ and suppose that we have a Separation Oracle $\mathsf{SEP}(\cK)$ for the set $\cK$ and a Subgradient Oracle $\mathsf{GRAD}(f)$ for $f$ that can be queried on any point $\x\in \cK.$ Then, a Separation Oracle for $\cK_f\coloneqq \epi(f) \cap  \{(\y,v)\colon v\leq b\}$ can be implemented with $1$ calls to $\mathsf{SEP}(\cK)$ and $\mathsf{GRAD}(f)$.
\end{lemma}
\begin{proof}[{\bf Proof}]
	Let $\z =(\x,v)\in \reals^{d+1}$. Using the available Oracles in lemma's statement, we will show that it is possible to build a Separation Oracle $\mathsf{SEP}(\cK_f)$ whose output $\tilde \c$ on input $\z$ requires only one call to the available Oracles. 
	
	Query $\mathsf{SEP}(\cK)$ with input $\x$. First, suppose that $\mathsf{SEP}(\cK)$ asserts that $\x \in \cK$. In this case, query $\mathsf{GRAD}(f)$ with input point $\x$ and let $(s,\g)$ be its output. In this case, we have $s=f(\x)$ and
	\begin{align}
	 s	+ \g^\top (\y - \x) \leq f(\y), \quad \text{for all }\y \in \reals^d. \label{eq:grad}
		\end{align}
If $s\leq v\leq b$, our desired Separation Oracle $\mathsf{SEP}(\cK_f)$ asserts that $(\x,v) \in \cK_f$. Otherwise, from \eqref{eq:grad} it follows that \begin{align}v\leq s \implies & (\g,-1)^\top (\x, v) \geq (\g,-1)^\top (\y, f(\y)) \geq  (\g,-1)^\top (\y, v'),  & \forall (\y,v')\in \cK_f, \label{eq:first} \\
v\geq b	\implies&  (\bm{0}, 1)^\top (\x,v) \geq 	(\bm{0}, 1)^\top (\y, v'), & \forall (\y,v')\in \cK_f. 
\end{align}
Now, suppose that $\mathsf{SEP}(\cK)$ does not assert that $\x \in \cK$, and let $\c$ be its output. In this case, we have $\c^\top \x \geq \c^\top \y$, for all $\y \in \cK$. Combining this with the fact that $(\y,v')\in \cK_f$ only if $\y \in \cK$, we get that 
\begin{align}
(\c,0)^\top (\x,v) \geq (\c,0)^\top(\y, v'), \quad   \forall (\y,v')\in \cK_f. \label{eq:second}
\end{align}
Thus, we can define $\mathsf{SEP}(\cK_f)$ to be the Oracle that given input $\z=(\x,v)$ asserts $\z\in \cK_f$ when $\x\in \cK$ and $s\leq v\leq b$. And, otherwise, outputs $\tilde{\bm{c}}$ such that  
\begin{align}
	\tilde{\bm{c}} \coloneqq \left\{ \begin{array}{ll} (\g,-1), & \text{if } \x\in \cK \text{ and } v\leq s,\\ (\bm{0},1), & \text{if }\x\in \cK \text{ and } b\leq v,\\ (\bm{c},0), & \text{if }\x \notin \cK,   \end{array}  \right. 
	\end{align}
where $s,b$, $\g$ and $\c$ are as above. 
\end{proof}
We base our proof of Proposition \ref{prop:eucproj} on the following reduction from Separation to Linear Optimization:
\begin{theorem}[Theorem 15 of \cite{lee2018efficient} rephrased]
	\label{thm:lee}
	Let $\cK\subset \reals^d$ be a convex set satisfying $\mathbb{B}(r) \subseteq \cK \subseteq \mathbb{B}(R)$ and let $\kappa \coloneqq R/r$. For any $\c \in \reals^d$ and $\varepsilon \in(0,1)$, there exists an algorithm that given input $(\c, \varepsilon)$ outputs $\x$ such that, with probability $1-\varepsilon$, 
	\begin{align}
\x\in  \mathbb{B}(\cK,R\varepsilon) \quad \text{and}\quad		\c^\top \x \leq \min_{\y\in \cK} \c^\top \y +R \varepsilon \|\c\|, 
	\end{align}  
Furthermore, the algorithm returns $\x$ with an expected running time of $O(d \cdot \mathrm{SEP}_{\delta}(\cK) \ln (d\kappa/\varepsilon) + d^3 \ln^{O(1)}(d\kappa/\varepsilon))$.
	\end{theorem}
Theorem \ref{thm:lee} leads to the following proposition:
\begin{proposition}
	\label{prop:lee}
		Let $\cK\subset \reals^d$ be a convex set satisfying $\mathbb{B}(r) \subseteq \cK \subseteq \mathbb{B}(R)$ and let $\kappa \coloneqq R/r$. For any $\c \in \reals^d$ and $\varepsilon \in(0,1)$, there exists an algorithm $\cA$ that given input $(\c, \varepsilon)$ outputs $\x$ such that, with probability $1-\varepsilon$, 
	\begin{align}
	\x\in  \mathbb{B}(\cK)\quad \text{and}\quad 	\c^\top \x \leq \min_{\y\in \cK} \c^\top \y +  R \varepsilon\|\c\|,  \label{eq:sit}
	\end{align}  
Furthermore, the algorithm returns $\x$ with an expected running time of $O(d \cdot \mathrm{SEP}_{\delta}(\cK) \ln (d\kappa/\varepsilon) + d^3 \ln^{O(1)}(d\kappa/\varepsilon))$.
\end{proposition}
The difference between Proposition \ref{prop:lee} and Theorem \ref{thm:lee} is that in the former the vector $\x$ is in $\ball(\cK)$ instead of $\ball(\cK, R \varepsilon)$.
\begin{proof}[{\bf Proof of Proposition \ref{prop:lee}}]
	Let $\x'$ be the output of the algorithm in Theorem \ref{thm:lee} given input $(\c, \varepsilon')$, where $\varepsilon' \coloneqq \varepsilon/(3\kappa)$. First, we note that $
		\min_{\y \in \cK} \c^\top \y= -\sigma_{\cK}(-\c) \leq -  r\|\c\|,$
	where the last inequality follows by \cite[Lemma 2]{mhammedi2021efficient}. This, together with the guarantee on $\x'$ from Theorem \ref{thm:lee} and the fact that $\varepsilon \le 1$, implies that with probability at least $1-\varepsilon$, \begin{align}\x' \in \ball(\cK, R\varepsilon')  \quad \text{and}\quad  \c^{\top} \x' \leq \min_{\y \in \cK}\c^\top \y + \varepsilon' R \|\c\|\leq - r \|\c\| +\varepsilon R/(3 \kappa)\leq 0. \label{eq:nega}
	\end{align} 
In what follows, we condition on the event that \eqref{eq:nega} holds. Now, by Lemma \cite[Lemma 10]{mhammedi2021efficient}, we can compute $\tilde \gamma$ such that $\gamma_{\cK}(\x')\leq \tilde \gamma \leq \gamma_{\cK}(\x')+ \varepsilon/3$. using at most $\lceil\log_2(48\kappa^2/\varepsilon)\rceil$ calls to $\mathsf{SEP}(\cK)$. We let $\x=\x'/\tilde \gamma$ be the output of Algorithm $\cA$ in the Corollary's statement given input $(\c,\varepsilon)$. We now show that $\x$ satisfies \eqref{eq:sit}. By the fact that $\tilde \gamma \geq \gamma_{\cK}(\x')$ and positive homogeneity of the Gauge function, we have $\gamma_{\cK}(\x) = \gamma_{\cK}(\x')/\tilde \gamma\leq 1$, and so $\x\in \cK$. On the other hand, by the fact that $\tilde \gamma \leq \gamma_{\cK}(\x')+ \varepsilon/3$ and the guarantee of $\x'$ in \eqref{eq:nega}, we have 
	\begin{align}
	R \varepsilon\|\c\|/3 +  \min_{\y \in \cK}\c^\top \y \geq \c^\top \x'  = \tilde \gamma   \c^\top \x   \stackrel{(*)}{\geq}  \gamma_{\cK}(\x')  \c^\top \x' - \varepsilon  |\c^\top \x| /3,  \label{eq:gen}
		\end{align}
	where $(*)$ follows by the fact that $\c^\top \x'\leq 0$ (see \eqref{eq:nega}) and $\tilde \gamma \leq \gamma_{\cK}(\x')+\varepsilon/3$. Now, since $\x' \in \mathbb{B}(\cK, R\varepsilon')$ (see \eqref{eq:nega}), there exists $\y \in\cK$ such that $\|\x'-\y\|\leq R\varepsilon'$. Using this and the subadditivity of the Gauge function, we have 
	\begin{align}
		\gamma_{\cK}(\x') \leq \gamma_{\cK}(\y) + \gamma_{\cK}(\x' - \y)\leq 1 + \|\x'-\y\|/r \leq 1 + \kappa \varepsilon'\leq 1+\varepsilon/3, \label{eq:toplug}
		\end{align}
	where the second inequality follows by the fact that $\mathbb{B}(r)\subseteq \cK$ and \cite[Lemma 2]{mhammedi2021efficient}. Plugging \eqref{eq:toplug} into \eqref{eq:gen} and using that $\c^\top \x \leq 0$ (see \eqref{eq:nega}), we get that
	\begin{align}
			R\varepsilon\|\c\|/3 +  \min_{\y \in \cK}\c^\top \y  \geq \c^\top \x -  |\c^\top \x| \varepsilon/3 -   |\c^\top \x| \varepsilon/3,
		\end{align}
Using this and the Cauchy Schwarz inequality, we get the desired inequality.
\end{proof}
We now use Proposition \ref{prop:lee} to give an upper bound on the complexity of optimizing a convex function.
\begin{proposition}
\label{prop:optimize}
Let $\cK$ be a closed convex set satisfying $\ball(r)\subseteq \cK \subseteq \ball(R)$, and let $f\colon \cK \rightarrow \reals$ be a closed convex function such that $\sup_{\y \in \cK} f(\y) - \inf_{\y \in \cK} f(\y)\leq B/2$, for some $B>0$. Then, for any $\varepsilon\in(0,1)$, there exists an algorithm that outputs $\x$ on input $\varepsilon$ s.t., with probability at least $1-\varepsilon$,
\begin{align}
 f(\x) \leq \inf_{\y \in \cK} f(\y) + \varepsilon \cdot \left(1+B\right),
	\end{align}
Furthermore, the algorithm outputs $\x$ with an expected running time of $$O(d (\mathrm{SEP}(\cK)+\mathrm{GRAD}(f))\ln (d\kappa B/\varepsilon) +d^3 \ln^{O(1)}(d\kappa B/\varepsilon)), \ \  \text{where}  \ \ \kappa\coloneqq \frac{R}{r}.$$
\end{proposition}
\begin{proof}[{\bf Proof}]
Let $g\colon \reals^d \rightarrow \reals \cup \{+\infty\}$ be the closed convex function defined by $g(\x) = f(R\x)$ if $\x \in \cK/R$ and $g(\x)=+\infty$, otherwise. Define the sets $\cK'= \dom(g)$ and $$\cK_g \coloneqq \epi(g) \cap \left\{(\y,v)\in \reals^d\times \reals \colon v\leq B \right\}.$$
Since $\sup_{\y' \in \cK'} g(\y') - \inf_{\y' \in \cK'} g(\y')=\sup_{\y \in \cK} f(\y) - \inf_{\y \in \cK} f(\y)\leq B/2$ and $\ball(1/\kappa)\subseteq \cK' \subseteq \ball(1)$, we have that $$\ball((\kappa\wedge (B/2))^{-1})\subseteq \cK_g \subseteq \ball(1\vee B).$$  Let $\cA$ be the algorithm in Proposition \ref{prop:lee} instantiated with $(\cK=\cK_g, r=(\kappa\wedge (B/2))^{-1}, R=1\vee B)$, and let $\z=(\x',v)$ be the output of $\cA$ given input $(\c, \varepsilon)$, where $\c = (\bm{0},1)$. Then, by Proposition \ref{prop:lee} and the fact that $\cK_g \subseteq \ball(1+B)$, we have, with probability at least $1-\varepsilon$, 
\begin{align}
	(\x', v) \in \cK_g \quad \text{and} \quad  v = \c^\top (\x',v) \leq \min_{(\y',s)\in \cK_g} \c^\top (\y',s) + \left(1+ B\right)\varepsilon. \label{eq:Ltoinst}
	\end{align}
Furthermore, $\cA$ outputs $\x$ with an expected running time of $$O(d \cdot \mathrm{SEP}(\cK_g)\ln (d\kappa B/\varepsilon) + d^3 \ln^{O(1)}(d\kappa B/\varepsilon))$$ which, by Lemma \ref{lem:switch}, is equal to  $O(d (\mathrm{SEP}(\cK)+\mathrm{GRAD}(f))\ln (d\kappa B/\varepsilon) +d^3 \ln^{O(1)}(d\kappa B/\varepsilon))$. For the rest of this proof, we condition on the event that \eqref{eq:Ltoinst} holds. Let $\y_*\in \argmin_{\y\in \cK'}g(\y)$ and $v_*\coloneqq g(\y_*)$, and note that $(\y_*,v_*)\in \cK_g$ and $g(\y_*)= \inf_{\y \in \cK}f(\y)$. Thus, instantiation \eqref{eq:Ltoinst} with $(\y', s)=(\y_*, v_*)$ we obtain that,
\begin{align}
	v \leq g(\y_*) + \varepsilon \cdot \left(1 + B  \right)= \inf_{\y \in \cK}f(\y) + \varepsilon \cdot \left(1 +B  \right).
	\end{align}
Now, the fact that $(\x',v)\in \cK_g$ implies that $\x' \in \cK/R$ and $g(\x')\leq v$. Thus, the vector $\x\coloneqq R\x'$ satisfies
\begin{align}
	f(\x) = g(\x') \leq v \leq \inf_{\y \in \cK}f(\y) + \varepsilon \cdot \left(1 + B  \right).
	\end{align} 
This completes the proof.
\end{proof}

	\begin{proof}[{\bf Proof of Proposition \ref{prop:eucproj}}]
	Let $f\colon \cU \rightarrow \reals$ be the convex function defined by $f(\u)\coloneqq \|\gamma \u - \w\|^2$, for $\gamma$, $\w$, and $\cU$ as in the lemma's statement. Note that for any point $\u \in \cU$, we have $\nabla f(\u)=2 \gamma (\gamma \u - \w)$. Thus, a Subgradient Oracle for $f$ on $\cK$ can be implemented with complexity $\mathrm{GRAD}(f)\leq O(d)$. Furthermore, by definition of the polar set $\K^\circ$, a Separation Oracle for the latter (and as a result for $\cU$) can be implemented with one call to an LOO for $\K$. So we have available a Separation Oracle $\mathsf{SEP}(\cU)$. 
	
	Since $\cU=\sqrt{1-\epsilon}\K^\circ \subseteq \K^{\circ} \subseteq \mathbb{B}(1/r)$ and $\gamma \in(0,b]$, we have $\left(\max_{\u\in \cU}f(\u)-\min_{\v\in \cU}f(\v)\right)\leq \frac{2b^2}{r^2}+ 2 \|\w\|^2$. Instantiating Proposition \ref{prop:optimize} with $\varepsilon =  \delta/(1+{4b^2}/{r^2}+ 4 \|\w\|^2)$ and $B = {4b^2}/{r^2}+ 4 \|\w\|^2$ yields an algorithm $\cA_{\mathsf{proj}}^{\cU}$ that outputs $\z\in \cU$ with our desired suboptimality and computational complexity guarantees.
	\end{proof}
	
	\subsection{Proof of Lemma \ref{lem:gs}}
	\label{sec:gs}
	
	\begin{algorithm}[t]
		\caption{Golden-section search for optimizing the function $\Theta$ in Lemma \ref{lem:oned}. $\cA_{\mathsf{GS}}(\w,\delta)$}
		\label{alg:gs}
		\begin{algorithmic}
			\REQUIRE Input point $\w\in \reals^d\setminus\{\bm{0}\}$ and $\delta \in(0,1)$. $\cA^{\cU}_{\mathsf{proj}}$ as in Proposition \ref{prop:eucproj}.
			\vspace{0.1cm}
			\STATE Set $\varphi \coloneqq (\sqrt{5}+1)/2$ and $K = \left\lceil-\log_{\varphi-1}  \frac{8 \kappa^2 \|\w\|^2 + 4 \kappa  \|\w\|^2}{ \delta}\right\rceil$. \algcomment{$\kappa=\frac{R}{r}$ with $r, R$ as in \eqref{eq:sandwitch}}
			\STATE Define $\tilde \Theta(\cdot)= (\cdot)^2 + r^2 \|\gamma \cA^{\cU}_{\mathsf{proj}}(\w,\cdot, \frac{\delta}{4 K \varphi})- \w\|^2$. 
			\STATE Set $\gamma_1= 0$, and $\mu_1=R\|\w\|$.
			\STATE Set $\bar \gamma_1 = \mu_1 - (\mu_1-\gamma_1) \varphi^{-1}$ and $\bar \mu_1= \gamma_1 + (\mu_1-\gamma_1)\varphi^{-1}$.
			\STATE Set $\tilde\theta_{\gamma, 1} = \tilde \Theta(\bar \gamma_1)$ and $\tilde\theta_{\mu,1}  = \tilde \Theta (\bar \mu_1)$
			\FOR{$i=1,\dots, K$} \label{line:while}
			\IF{$\tilde\theta_{\gamma, i} < \tilde\theta_{\mu, i}$}
			\STATE Set $\gamma_{i+1} = \gamma_i$ and $\mu_{i+1} = \bar \mu_i$.
			\STATE Set $\bar \gamma_{i+1} = \mu_{i+1} - (\mu_{i+1}-\gamma_{i+1}) \varphi^{-1}$ and $\bar \mu_{i+1} = \bar\gamma_{i}$.
			\STATE Set $\tilde\theta_{\gamma, i+1} = \tilde \Theta(\bar \gamma_{i+1})$ and $\tilde\theta_{\mu,i+1}  =\tilde\theta_{\gamma,i}$.
			\ELSE 
			\STATE Set $\gamma_{i+1} = \bar \gamma_i$ and $\mu_{i+1} = \mu_i$.  
			\STATE Set $\bar \gamma_{i+1} = \bar \mu_i$ and $\bar \mu_{i+1} = \gamma_{i+1} + (\mu_{i+1}-\gamma_{i+1})\varphi^{-1}$.
			\STATE Set $\tilde\theta_{\gamma, i+1} = \tilde\theta_{\mu, i}$ and $\tilde\theta_{\mu,i+1}  = \tilde \Theta (\bar \mu_{i+1})$
			\ENDIF  
			\ENDFOR 
			\STATE Return $\hat \gamma = (\gamma_K+\mu_K)/2$.
		\end{algorithmic}
	\end{algorithm}	
	For the proof of Lemma \ref{lem:gs}, we need the following result on the Golden-search algorithm in Alg.~\ref{alg:gs} :
	\begin{lemma}
		\label{lem:subgold}
		Let $(\tilde\theta_{\gamma, i},\tilde\theta_{\mu, i}, \bar \gamma_i, \bar \mu_i, \gamma_i, \mu_i, \varphi,K)$ be as in Algorithm \ref{alg:gs}, and for $\epsilon>0$ define the event   
		\begin{align}
			\cE \coloneqq \{	|\tilde\theta_{\gamma, i}- \Theta(\bar \gamma_i)|	\vee 	|\tilde\theta_{\mu, i}- \Theta(\bar \mu_i)| \leq \epsilon, \quad \forall i \in[K]\}.\label{eq:event4}
		\end{align}
		If $\cE$ holds, then $\min_{\gamma \in [\gamma_i, \mu_i]}\Theta(\gamma) \leq \min_{\gamma \geq 0}\Theta(\gamma) + 2 (i-1)\varphi\epsilon$, for all $i \in [K]$.
	\end{lemma}
	\begin{proof}[{\bf Proof}]
		We proceed by induction. By Lemma \ref{lem:optCeps}, we have $\lambda_* =\sigma_{\K_{\epsilon}}(\w)$, which together with the fact that $\K_{\epsilon}\subseteq \K \subseteq \cB(R)$ implies that $\gamma_* \leq \lambda_* \leq  R\|\w\|$, where $\gamma_* \in \argmin_{\gamma \in \reals}\Theta(\gamma)$. Thus, the claim of the lemma holds for $i=1$ since $\gamma_1 =0$ and $\mu_1 =R\|\w\|$. Now suppose the claim holds for $i$, and we will show that it holds for $i+1$. Assume without loss of generality that $\tilde \theta_{\mu,i}> \tilde \theta_{\gamma, i}$. In this case, we have $\mu_{i+1}=\bar \mu_i$ and $\gamma_{i+1}= \gamma_i$, and by \eqref{eq:event4}, we have 
		\begin{align}
			\Theta(\bar \gamma_{i}) < \Theta(\bar \mu_{i}) + 2 \epsilon. \label{eq:warumnicht}
		\end{align}
		Let $\nu_i \in \argmin_{\nu \in [\gamma_i, \mu_i]} \Theta(\nu)$. If $\nu_i \in [\gamma_i, \bar \mu_{i}]$, then $\min_{\gamma \in [\gamma_{i+1}, \mu_{i+1}]}\Theta(\gamma) \leq \min_{\gamma \in [\gamma_i, \mu_i]}\Theta(\gamma)\leq \min_{\gamma\geq 0}\Theta(\gamma) + 2 i\varphi\epsilon$, by the induction hypothesis. Now, suppose that $\nu_i \in (\bar \mu_i, \mu_i]$.
		
		By convexity of $\Theta$ (Lemma \ref{lem:oned}), we have $ \Theta(\bar \mu_i) + \Theta'(\bar \mu_i)\cdot (\bar \gamma_i- \bar\mu_i) \leq \Theta(\bar \gamma_i)$. Plugging this into \eqref{eq:warumnicht}, we get that $
		\Theta'(\bar \mu_i) \cdot (\bar \gamma_{i} - \bar \mu_{i}) \leq 2 \epsilon.$
		This, together with the fact that $\Theta'(\bar \mu_{i})<0$ (since $\nu_i \in (\bar \mu_{i}, \mu_{i}]$ and $\Theta$ is convex) implies that \begin{align}|\Theta'(\bar \mu_{i})|\cdot (\bar \mu_{i}- \bar \gamma_{i}) \leq 2 \epsilon.
		\end{align} 
		Combining this with the fact that $\Theta(\bar \mu_{i}) \leq 	\Theta'(\bar \mu_{i}) \cdot ( \bar \mu_{i}- \nu_i) + \Theta(\nu_i)$ implies that 
		\begin{align}
			\Theta(\bar \mu_{i}) \leq 2 \epsilon \varphi +  \Theta(\nu_i), \label{eq:coming}
		\end{align}
		where we used that $(\nu_i- \bar \mu_{i}) \leq (\mu_i - \bar \mu_i)\leq \varphi \cdot (\bar \mu_i - \bar \gamma_i)$. Using \eqref{eq:coming}, together with $\Theta(\nu_i) = \min_{\gamma \in[\gamma_i, \mu_i]} \Theta(\gamma)\leq \min_{\gamma \geq 0}\Theta(\gamma) + 2 (i-1) \varphi \epsilon$ (by the induction hypothesis), we get that 
		\begin{align}
			\min_{\gamma \in[\gamma_{i+1}, \mu_{i+1}]} \Theta(\gamma)\leq \Theta(\mu_{i+1})=\Theta(\bar\mu_{i}) \leq \min_{\gamma \geq 0}\Theta(\gamma) + 2 i \varphi \epsilon,
		\end{align}
		which completes the induction.
	\end{proof}
	
	\begin{proof}[{\bf Proof of Lemma \ref{lem:gs}}]
		Let $(\tilde\theta_{\gamma, i},\tilde\theta_{\mu, i}, \bar \gamma_i, \bar \mu_i, \gamma_i, \mu_i)$ be as in Algorithm \ref{alg:gs}. Further, let $\Theta\colon \gamma^2 + r^2 \inf_{\u \in \sqrt{1-\epsilon}\cU} \|\lambda \u - \w\|^2$ and $\gamma_* \in \argmin_{\gamma \geq 0} \Theta(\gamma)$. By Proposition \ref{prop:eucproj} and a union bound, we have that with probability at least $1-\delta$, 
		\begin{align}
			|\tilde\theta_{\gamma, i}- \Theta(\bar \gamma_i)|	\vee 	|\tilde\theta_{\mu, i}- \Theta(\bar \mu_i)| \leq r^2 \delta', \quad \forall i \in[K],\label{eq:event}
		\end{align}
		where $\delta' = \delta/(4K \varphi)$. Let $\cE$ be the event in \eqref{eq:event}. For the rest of this proof, we condition on the event $\cE$. By Lemma \ref{lem:subgold}, we have 
		\begin{align}
			\min_{\gamma\in[\gamma_K, \mu_K]} \Theta(\gamma)\leq 2K r^2 \varphi \delta'  + \Theta(\gamma_*).\label{eq:religion}
		\end{align}
		Now, by the fact that $\Theta$ is convex, $(4R \|\w\| + 2r \|\w\|)$-Lipschitz on $(0, R\|\w\|]$ (Lemma \ref{lem:oned}), and $[\gamma_K, \mu_K] \subseteq [0, R\|\w\|]$, we have 
		\begin{align}
			\Theta(\hat \gamma)& \leq  	\min_{\gamma\in[\gamma_K, \mu_K]} \Theta(\gamma) + (4R \|\w\| + 2r \|\w\|) \cdot (\mu_K - \gamma_K), \nn \\
			&  \leq  2K r^2 \varphi \delta'  + (4R \|\w\| + 2r \|\w\|) \cdot (\varphi-1)^{K} + \Theta(\gamma_*), \label{eq:og}
		\end{align}
		where the last inequality follows by \eqref{eq:religion} and the fact that $\mu_i- \gamma_i\leq  (\varphi-1)(\mu_{i-1}- \gamma_i)$, for all $i\in[2..K]$, and $\mu_1 - \gamma_1 =R \|\w\|$. Substituting the expression of $K$ from Algorithm \ref{alg:gs} into \eqref{eq:og}, leads to the desired result. 
	\end{proof}

	\subsection{Proof of Lemma \ref{lem:linopt}}
	\label{sec:linopt}
	\begin{proof}[{\bf Proof of Lemma \ref{lem:linopt}}]
		Let $\epsilon$, $\delta\in(0,1) $, $\w\in \reals^d\setminus \{\bm{0}\}$, and $\cU \coloneqq \sqrt{1-\epsilon} \K^\circ$. Consistent with Algorithm \ref{alg:LOOCeps}, we let $\bar \w \coloneqq \w/\|\w\|$ and note that 
		\begin{align}
			\argmin_{\v\in \K_{\epsilon}} \inner{\v}{\w} = \argmin_{\v\in \K_{\epsilon}} \inner{\v}{\bar \w}. \label{eq:argmineq}
		\end{align}
		Further, define $S(\gamma)\coloneqq\inf_{\u \in \cU}\|\gamma \u - \bar \w\|$, $\Theta\colon \gamma \mapsto \gamma^2 + r^2 S(\gamma)^2$, and let $\u_{\gamma}\in \argmin_{\u \in \cU} \|\gamma \u - \bar \w\|^2$. When $\gamma =\gamma_*$, we have $\u_* = \u_{\gamma}$, where $\gamma_*$ and $\u_*$ are as in Lemma \ref{lem:optCeps} with $\w$ replaced by $\bar \w$; that is  
		\begin{align}
			(\gamma_*, \u_*)\in \argmin_{\gamma\geq 0, \u \in \sqrt{1-\epsilon}\K^\circ} \left\{\gamma^2 + r^2 \|\gamma \u - \bar\w\|^2 \right\}.
		\end{align}
		Also, consistent with Lemma \ref{lem:optCeps}, we let $\lambda_* \coloneqq \sqrt{\gamma^2_* + r^2 \|\gamma \u_* - \bar\w\|^2}$.
		Throughout, we will use the fact that $\lambda_* =\sigma_{\K_{\epsilon}}(\bar{\w})$ (follows from Lemma \ref{lem:optCeps}), which together with the fact that $\K_{\epsilon}\subseteq \K \subseteq \cB(R)$ implies
		\begin{align}
			\gamma_* \leq \lambda_* \leq  R\|\bar{\w}\|=R.\label{eq:bound}
		\end{align}
		We first show that $\Delta \coloneqq \|\gamma_* \hat \u  - \bar{\w}  \|^2 - \|\gamma_* \u_{*} - \bar{\w} \|^2\leq O(\delta^{1/2})$ ($\hat\u$ defined in Alg.~\ref{alg:LOOCeps}).  We have 
		\begin{align}
			\Delta & =\underbrace{\|\gamma_* \hat \u  - \bar{\w}  \|^2-  \|\hat \gamma \hat \u  - \bar{\w}  \|^2}_{A} + \underbrace{\hat \gamma^2/r^2 +  \|\hat \gamma  \u_{\hat \gamma}  - \bar{\w}  \|^2 - ( \gamma^2_*/r^2 +  \| \gamma_*  \u_{*}  - \bar{\w}  \|^2)}_{B}  \nn \\ & \qquad + \underbrace{ \gamma^2_*/r^2 -\hat \gamma^2/r^2}_{C}  + \underbrace{\|\hat \gamma \hat \u  - \bar{\w}  \|^2 -  \|\hat \gamma \u_{\hat \gamma}  - \bar{\w}  \|^2}_{D}.\label{eq:terms}
		\end{align}
		We now bound each term on the RHS of \eqref{eq:terms}. By Lemma \ref{lem:gs} and the fact that $\Theta$ is $1$-strongly convex (Lemma \ref{lem:oned}), we have with probability at least $1-\delta$,
		\begin{align}
			| \hat \gamma -  \gamma_*|^2/2 \leq |\Theta(\hat \gamma) - \Theta(\gamma_*)| \leq r^2\delta. \label{eq:approxd}
		\end{align}
		Let $\cE_1$ be the event that \eqref{eq:approxd} holds. Proposition \ref{prop:eucproj} guarantees that $\hat\u\in \cU$ and so $\|\hat \u\|\leq 1/r$. Using this and \eqref{eq:approxd}, we can bound the terms $A, B$, and $C$ in \eqref{eq:terms} in terms of $b\coloneqq 2 R + (2\delta)^{1/2} r \geq \hat \gamma + \gamma_*$ (the inequality follows by \eqref{eq:bound} and \eqref{eq:approxd}) under the event $\cE_1$ as 
		\begin{gather}
			A = (\gamma^2_* - \hat \gamma^2) \|\hat \u\|^2 -2(\gamma_* - \hat \gamma)\inner{\hat \u}{\bar{\w}}  \leq (2\delta)^{1/2} b/r  +(8\delta)^{1/2} , \label{eq:event0} \\ B = \frac{\Theta(\hat\gamma) - \Theta(\gamma_*)}{r^2} \leq  \delta, \ \ C= (\gamma_*- \gamma)(\hat\gamma +\gamma_*)/r^2\leq \frac{ b (2\delta)^{1/2}}{r},  \label{eq:event1}
		\end{gather}
		Proposition \ref{prop:eucproj} immediately implies that 
		\begin{align}
			D \leq \delta. \label{eq:event2}
		\end{align}
		with probability at least $1-\delta$. Let $\cE_2$ be the event that \eqref{eq:event2} holds. Define the event $\cE \coloneqq \cE_1 \cap \cE_2$ and note that $\mathbb{P}[\cE]\geq 1 - 2\delta$ via a union bound. In what follows, we assume that $\cE$ holds. By \eqref{eq:terms}, \eqref{eq:event0}, \eqref{eq:event1} and \eqref{eq:event2}, we have \begin{align}
			\|\gamma_* \hat \u  - \bar{\w}  \|^2 - \inf_{\v \in \gamma_* \cU} \|\v - \bar{\w} \|^2= \|\gamma_* \hat \u  - \bar{\w}  \|^2 - \|\gamma_* \u_{*} - \bar{\w} \|^2 \leq C_{\delta, \w}, \label{eq:poem}
		\end{align} 
		where $C_{\delta, \w}\coloneqq 2\delta  +(8\delta)^{1/2}\left( \frac{b}{r}+1\right)$.
		Since $\hat \u \in \cU$ (and hence $\gamma_* \hat \u \in \gamma_* \cU$) and the function $\v \mapsto \|\v - \bar{\w}\|^2$ is 1-strongly convex on $\gamma_* \cU$, \eqref{eq:poem} implies that 
		$\|\gamma_* \hat\u  - \gamma_* \u_*\|^2\leq 2C_{\delta, \w}$. Using this and \eqref{eq:approxd}, we have 
		\begin{align}
			\|\hat \gamma \hat \u - \gamma_* \u_*\|^2 &\leq 2 C_{\delta, \w} + (\hat \gamma^2 - \gamma_*^2)\|\hat \u\|^2 - 2 (\hat \gamma -\gamma_*) \gamma_* \inner{\hat \u}{\u_*}\nn \\  & \leq C'_{\delta, \w}\coloneqq 2C_{\delta, \w}+ (2\delta)^{1/2} b/r  +(8\delta)^{1/2}\kappa , \label{eq:event3}
		\end{align}
		where we used the fact that $\gamma_*\leq R, \|\hat \u\|\vee \|\u_*\|\leq 1/r$, and $\hat \gamma-\gamma_*\leq (2\delta)^{1/2}r$.
		We now use \eqref{eq:event3} to show the claim of the lemma. Let $\hat\z \coloneqq \bar{\w}- \hat \gamma \hat \u$ and $\z_* \coloneqq \bar{\w} - \gamma_* \u_*$. 
		From \eqref{eq:event3}, we have $
		|\inner{\bar{\w}}{ \hat \z} - \inner{\bar{\w}}{\z_*}|^2 \leq C_{\delta, \w}' $, and so since $\inner{\bar{\w}}{\z_*}\geq \nu \coloneqq \frac{\epsilon}{2 \kappa^3 (1+\kappa^2\epsilon)^{3/2}}$ (Lem.~\ref{lem:optCeps}), we have 
		\begin{align}
			\frac{1}{\inner{\bar{\w}}{\hat \z}}\leq  \frac{1}{\inner{\bar{\w}}{\z_*}- \sqrt{C_{\delta, \w}'} } \leq \frac{1}{\inner{\bar{\w}}{\z_*}} \left(1+\frac{2\sqrt{C_{\delta, \w}'}  }{\inner{\bar{\w}}{\z_*}}\right) \leq \frac{1}{\inner{\bar{\w}}{\z_*}} + \frac{2\sqrt{C_{\delta, \w}'} }{\nu^2},\label{eq:frist}
		\end{align}
		where the second inequality follows by the fact that $1/(1-x)\leq 1 + 2 x$, for all $x\in[0,1/2]$, and the fact that 
		\begin{align}
			\frac{{C_{\delta, \w}'} }{\inner{\bar{\w}}{\z_*}^2} \leq  	\frac{{C_{\delta, \w}'}}{\nu^2}\leq \frac{64 \sqrt{\delta } \kappa ^{12} (7 \sqrt{\delta }+8 \sqrt{2} \kappa  )}{ \epsilon ^2}  \leq 1/4,
		\end{align} 
		by the range restriction on $\delta$. Also from \eqref{eq:event3}, we immediately get  
		\begin{align}
			\|\hat\z-\z_* \|^2 	\leq C_{\delta, \w}'. \label{eq:iso}
		\end{align}
		Let $\lambda_*^2 \coloneqq  \gamma^2_* + r^2 \| \gamma_*  \u_* - \bar{\w} \|^2$ and $\hat \lambda \coloneqq  \hat \gamma^2 + r^2 \|\hat \gamma \hat \u - \bar{\w} \|^2$. By \eqref{eq:event1} and \eqref{eq:event2}, we also have 
		\begin{align}
			r^2\delta +(2\delta)^{1/2} b r \geq r^2(B+ D) \geq 	|\hat \lambda^2 - \lambda^2_*|\geq |\hat \lambda - \lambda_*|\cdot  \lambda_* \geq |\hat \lambda - \lambda_*| \frac{r}{\sqrt{1+\kappa^2\epsilon}} ,  \label{eq:lambdabound}
		\end{align}
		where the last inequality follows by the fact that $\lambda_* = \sigma_{\K_{\epsilon}}(\bar{\w})$ and $\mathbb{B}(r)/\sqrt{1+\kappa^2 \epsilon}  \subseteq \K_{\epsilon}$, which implies that $\lambda_*  \geq \|\bar{\w}\|r/\sqrt{1+\kappa^2\epsilon}$ (see proof of Lemma \ref{lem:optCeps}). Rearranging \eqref{eq:lambdabound}, we get 
		\begin{align}
			|\hat \lambda - \lambda_*|  \leq \frac{(r \delta + (2\delta)^{1/2}b)\sqrt{1+\kappa^2\epsilon}}{\|\bar{\w}\|} \leq \tilde C_{\delta, \w} \coloneqq \frac{6 \sqrt{\delta } (\kappa +1) R (\|\bar{\w}\|+1)}{\|\bar{\w}\|}, \label{eq:deficit}
		\end{align}
		where the last inequality follows by definition of $b$ and the fact that $\epsilon,\delta \in(0,1)$. Let $\v_*\coloneqq \lambda_*/\inner{\z_*}{\bar \w}$ ($\inner{\z_*}{\bar \w}>0$ by Lemma \ref{lem:optCeps}), then by Lemma \ref{lem:optCeps} and \eqref{eq:argmineq}, we have 
		\begin{align}
			\v_* &\in \argmin_{\v \in \K_\epsilon}\inner{\v}{\w}, \label{eq:belong} \\
			\text{and}	\quad 		\|\hat \v -\v_*\| \nn  =&	\left|\frac{\hat \lambda \hat \z}{\inner{\hat \z}{ \bar{\w}}} - \frac{\hat \lambda \hat \z}{\inner{\z_*}{ \bar{\w}}} +\frac{\hat \lambda \hat \z }{\inner{\z_*}{\bar{\w}}} - \frac{\hat \lambda \z_* }{\inner{\z_*}{\bar{\w}}} + \frac{\hat \lambda  \z_* }{\inner{\z_*}{\bar{\w}}}  -\frac{\lambda_*  \z_* }{\inner{\z_*}{\bar{\w}}} \right|, \nn \\
			\leq & \frac{2\sqrt{C'_{\delta, \w}} }{\nu^2} \|\hat \lambda  \hat \z\| + \frac{|\hat \lambda|}{\inner{\z_*}{\bar{\w}}} \|\hat \z - \z_*\| + \frac{\|\z_*\|}{ \inner{\z_*}{\bar{\w}}} |\hat \lambda - \lambda_*|,  \nn
		\end{align}
		where the last step follows by the triangle inequality and by \eqref{eq:frist}. Now, by the facts that $\|\hat \z\|\leq \|\bar{\w} \|+ \gamma_* \|\hat \u\|   +|\hat \gamma -\gamma_*|\|\hat \u\|\leq (1+\kappa) + (2\delta)^{1/2} r$, $\|\hat \u\|\leq 1/R$, and $|\hat \lambda|\leq \lambda_* + \tilde C_{\delta, \w}\leq R + \tilde C_{\delta, \w}$ (by \eqref{eq:deficit} and \eqref{eq:bound}), we get	\begin{align}
			\|\hat \v -\v_*\|&	\leq  \frac{2\sqrt{C'_{\delta, \w}} }{\nu^2} (R+ \tilde C_{\delta, \w}) ((1+\kappa)  +(2\delta)^{1/2})  + \frac{(R+ \tilde C_{\delta, \w}) \cdot\sqrt{ C_{\delta, \w}'}}{\nu^2} +\frac{(1+\kappa) \tilde C_{\delta, \w}}{\nu^2}, \quad \nn \\
			&\leq 72^2 R \delta^{1/4}\kappa^3 /\nu^2,\label{eq:free}\\
			&	\leq   \frac{405^2 R \delta^{1/4} \kappa^{15} }{\epsilon^2},\label{eq:origin} 
		\end{align}
		where to get to \eqref{eq:free} we simplified expressions using that $\kappa\geq 1$ and $\delta \leq 1$, and the last inequality we used that $\nu =  \epsilon/(2 \kappa^3 (1+\kappa^2\epsilon)^{3/2}) \geq \epsilon/(2^{5/2} \kappa^6) $. Now, by sub-additivity of the Gauge function and the fact that  $\gamma_{\K_{\epsilon}}(\cdot)\leq \sqrt{1+\epsilon \kappa^2}\|\cdot\|/r$ (since $\mathbb{B}(r)\subseteq \sqrt{1+\epsilon \kappa^2}\K_{\epsilon}$), we have 
		\begin{align}
			\gamma_{\K_{\epsilon}}(\hat \v) \leq \gamma_{\K_{\epsilon}}(\v_*)+  \gamma_{\K_{\epsilon}}(\hat \v-\v_*)& \leq 1 + \sqrt{1+\epsilon\kappa^2}\|\hat \v-\v_*\|/r, \quad \text{(since $\v_*\in \K_{\epsilon}$ by \eqref{eq:cold})}\nn \\ &  \leq 1 + (1+\kappa)\|\hat \v-\v_*\|/r,\nn \\
			& \leq 1 +   \frac{ 576^2 \delta^{1/4} \kappa^{17}  }{\epsilon^2}. \label{eq:scale}
		\end{align}
		This implies that $	\gamma_{\K_{\epsilon}}(\tilde \v) \leq 1$, where $\tilde \v$ is as in Algorithm \ref{alg:LOOCeps}. This implies that $\tilde \v \in \K_\epsilon$. Now using \eqref{eq:free}, the triangle inequality and the fact that $(1-1/(1+x))\leq x$, for all $x\geq 0$, we get 
		\begin{align}
			\|\tilde \v - \v_*\| & \leq	\|\tilde \v - \hat \v\|  + 	\|\hat \v - \v_*\| ,\nn \\
			& = \left(1-\frac{1}{1 + 576^2  { \delta^{1/4} \kappa^{17}  }{\epsilon^{-2}}}\right)  \|\tilde \v\| + \|\hat \v - \v_*\|, \nn \\
			& \leq   \frac{ 576^2\delta^{1/4} \kappa^{17} }{\epsilon^{2}}  \|\hat \v\|  + \|\hat \v - \v_*\|, \nn \\
			&  \leq \frac{ 576^2\delta^{1/4} \kappa^{17}  }{\epsilon^{2}}   \|\hat \v- \v_*\| + \frac{ 576^2\delta^{1/4} \kappa^{17} }{\epsilon^{2}}   \|\v_*\|   + \|\hat \v - \v_*\|,\nn \\
			& \leq \frac{484^4 R \delta^{1/4}  \kappa^{32} }{\epsilon^4}  +\frac{ 576^2R\delta^{1/4} \kappa^{17}  }{\epsilon^{2}},\label{eq:foll} \\
			& \leq \frac{484^4 R \delta^{1/4}  \kappa^{32} }{\epsilon^4} , \quad (\kappa \geq 1 \text{  and  } \epsilon \in(0,1)) \label{eq:last}
		\end{align}
		where \eqref{eq:foll} follows by \eqref{eq:origin} and the fact that $\|\v_*\|\leq R$ ($\v_* \in \K_{\epsilon} \subseteq \K \subseteq \mathbb{B}(R)$ by \eqref{eq:belong} and Lemma \ref{lem:approxlemma}). Both \eqref{eq:scale} and \eqref{eq:last} hold under the event $\cE$, which satisfies $\P[\cE]\geq 1 -\delta$. This together with \eqref{eq:belong} shows the first claim of the lemma, i.e.~\eqref{eq:target}. The upper bound on the number of Oracle calls follows from Proposition \ref{prop:eucproj}, Lemma \ref{lem:gs}, and the fact that $\|\bar \w\|=1$.
	\end{proof}
	
	\subsection{Proof of Theorem \ref{thm:master2}}
	\begin{proof}[{\bf Proof of Theorem \ref{thm:master2}}]
		Let $\u\in \K$ and $\u' \coloneqq \u/\sqrt{1+\kappa^2\epsilon}$. Further, let $\cA_*$ be the instance of Alg.~\ref{alg:EffAlg} with $\cB \equiv \cA_{\FTL}(\K_{\epsilon})$ ($\cA_{\FTL}(\K_{\epsilon})$ is Alg.~\ref{alg:boundarywolf} with $\K = \K_{\epsilon}$), where we require the iterates $(\w_t)$ of the latter (i.e.~the instance of $\cA_{\FTL}(\K_{\epsilon})$) to satisfy \begin{align}
			\w_{t+1}=\bm{0}\quad \text{whenever}\quad  \sum_{s=1}^{t}\g_s = \bm{0}, \quad \forall t\in[T-1].  \label{eq:require}
		\end{align}
		Note that this is a valid choice for the iterates $(\w_{t})$ of $\cA_{\FTL}(\K_{\epsilon})$ since $\bm{0} \in \argmin_{\w\in \K_{\epsilon}}\sum_{s=1}^t \g_s^\top \w$ whenever $\sum_{s=1}^{t}\g_s=\bm{0}$. Then, since $\K_{\epsilon}$ is $\frac{2\epsilon}{r\sqrt{1+\epsilon \kappa^2}}$-strongly convex and $\u' \in \K_{\epsilon}$ (see Lemma \ref{lem:approxlemma}), Theorem~\ref{thm:master} implies that the outputs $(\x'_t)$ of $\cA_*$ in response to $(f'_t\colon \w \mapsto \inner{\w}{\g_t})$ with $(\g_t \in \partial f_t(\x_t))$ satisfy
		\begin{align}
			\frac{\ln (T+1) }{\eta} + \frac{\eta R^2}{4}V_T + \wtilde{O}\left( R \sqrt{ (1+\kappa^2 \epsilon)^{1/2}\frac{ V_T}{2 \epsilon}}\right) &  \geq 	\sum_{t=1}^T \inner{\g_t}{\x_t'-\u'} \nn \\ &\geq  \sum_{t=1}^T \inner{\g_t}{\x_t'-\u} +  \inner{\G_T}{\u}- \frac{\inner{\G_T}{\u}}{\sqrt{1+\kappa^2 \epsilon}},\nn \\
			& \geq \sum_{t=1}^T \inner{\g_t}{\x_t'-\u} - L R \kappa^2 \epsilon T, \label{eq:silver}
		\end{align}
		where $\G_T\coloneqq \sum_{t=1}^T \g_t$. Let $\cE$ be the event that $\{\forall t\in[T], \|\x_t - \x_t'\| \leq R/T, \ \ \text{and}  \ \ \x_t \in \K\}$. By our choice of $\delta$ in the setting of the theorem, Lemma~\ref{lem:linopt}, the requirement on $\cA_*$ in \eqref{eq:require}, and the fact that $\x_t =\bm{0}$ whenever $\sum_{s=1}^{t-1}\g_s =\bm{0}$ (see Alg.~\ref{alg:LOOCeps}), we have that $\mathbb{P}[\cE]\geq 1 - \rho/T$. Furthermore, under the event $\cE$, we have $ \inner{\g_t}{\x_t'-\x_t} \leq LR/T$. Combining this with \eqref{eq:silver} implies the desired result.
	\end{proof}

		\section{Euclidean Projection using a Linear Optimization Oracle}
	\label{app:projection}
	\newtheorem{facts}[theorem]{Facts}
	In this section, we argue that Euclidean projection onto $\K$ can be done using $\wtilde{O}(d^2)$ calls to an LOO for $\K$. Euclidean projection of a point $\w$ onto $\K$ is the problem of finding $\x_* \in \argmin_{\x \in \K} f(\x)$, where $f(\x) \coloneqq \|\x - \w\|^2$. Since the gradient $f$ at a point $\x$ can be computed in closed form as $\nabla f(\x) = 2 (\x-\w)$, a Subgradient Oracle for $f$ can be implemented with complexity $O(d)$. Thus, Proposition \ref{prop:optimize} implies that there exists an algorithm that finds $\x_* \in \argmin_{\x \in \K} f(\x)$ (i.e.~performs Euclidean projection) with an expected running time of 
	\begin{align}O(d \cdot \mathrm{SEP}(\K)\ln (d\kappa \|\w\|/\varepsilon) +d^3 \ln^{O(1)}(d\kappa \|\w\|/\varepsilon)),\label{eq:euc}
	\end{align}
	where $\mathrm{SEP}(\K)$ is the computational complexity of a Separation Oracle for $\K$. We now bound the complexity of a (weak) Separation Oracle for $\K$ in terms of the number of calls to an LOO for $\K$. For this, we use the following facts:
	\paragraph{Facts.} {\em For a closed convex set $\cK\subset \reals^d$, the following is true:
		\begin{enumerate}
			\item A Separation Oracle for $\cK^\circ$ can be implemented with one call to a Linear Optimization Oracle for $\cK$ (follows by definition of the polar). 
			\item Linear Optimization on $\cK$ can be implemented using $\wtilde{O}(d)$ expected \# of calls to a Separation Oracle for $\cK$ and an additional $\wtilde{O}(d^3)$ arithmetic operations (see Proposition \ref{prop:lee}).
	\end{enumerate}}
	Instantiating Fact 1 with $\cK = \K$, implies that we can implement a Separation Oracle for $\K^\circ$ with one call to an LOO for $\K$. Using this and instantiating Fact 2 with $\cK = \K^\circ$, we get that we can implement an LOO for $\K^\circ$ with $\wtilde{O}(d)$ expected \# of calls to an LOO for $\K$ and an additional $\wtilde O(d^3)$ arithmetic operations. Now, instantiating Fact 1 again, with $\cK= \K^\circ$ and using that $(\K^\circ)^\circ=\K$, we get that we can implement a Separation Oracle for $\K$ with $\wtilde{O}(d)$ expected \# calls to an LOO for $\K$ and an additional $\wtilde O(d^3)$ arithmetic operations. Combining this with the expected complexity upper bound of Euclidean projection in \eqref{eq:euc} implies that Euclidean projection can be performed with $\wtilde{O}(d^2)$ expected \# of calls to an LOO for $\K$ and an additional $\wtilde O(d^4)$ arithmetic operations.

\end{document}